\documentclass[10pt]{article}
\usepackage{geometry}
\usepackage{times}
\usepackage[T1]{fontenc}    

\usepackage{amsmath,amsfonts,bm}

\def\eqref#1{equation~\ref{#1}}

\def\1{\bm{1}}

\DeclareMathAlphabet{\mathsfit}{\encodingdefault}{\sfdefault}{m}{sl}
\SetMathAlphabet{\mathsfit}{bold}{\encodingdefault}{\sfdefault}{bx}{n}

\newcommand{\E}{\mathbb{E}}

\DeclareMathOperator*{\argmin}{arg\,min}

\usepackage{hyperref}
\usepackage{url}

\usepackage{enumitem}

\usepackage[utf8]{inputenc} 
\usepackage{booktabs}       
\usepackage{nicefrac}       
\usepackage{microtype}      
\usepackage{xcolor}         
\usepackage{colortbl}       
\colorlet{routingrowcolor}{blue!8}  
\usepackage{tabularx}
\newcommand{\routingtopgap}{\rule{0pt}{\dimexpr\ht\strutbox+\belowrulesep\relax}}
\newcommand{\routingbotgap}{\rule[-\dimexpr\dp\strutbox+\aboverulesep\relax]{0pt}{0pt}}
\usepackage{graphicx}
\usepackage{subcaption}
\usepackage{comment}
\usepackage{natbib}

\usepackage{standalone} 
\usepackage{tikz}
\usepackage{tcolorbox}\tcbuselibrary{breakable,skins} 
\newtcolorbox{promptbox}{enhanced,breakable,arc=4pt,boxrule=0.5pt,colback=black!5,colframe=black!45,colbacktitle=black!12,coltitle=black,fonttitle=\bfseries\small,title=Prompt,left=4pt,right=4pt,top=3pt,bottom=3pt}
\newtcolorbox{responsebox}{enhanced,breakable,arc=4pt,boxrule=0.5pt,colback=green!6,colframe=green!45!black,colbacktitle=green!15,coltitle=black,fonttitle=\bfseries\small,title=Student response,left=4pt,right=4pt,top=3pt,bottom=3pt}
\usetikzlibrary{arrows.meta, shapes.geometric, positioning, decorations.pathreplacing, calc}

\usepackage{amsmath, amssymb, amsfonts, amsthm}
\usepackage{cleveref}
\usepackage{physics}
\usepackage{algorithm}
\usepackage{algorithmic}
\newtheorem{theorem}{Theorem}
\newtheorem{corollary}{Corollary}
\newtheorem{lemma}{Lemma}

\newtheorem{proposition}{Proposition}

\newtheorem{assumption}{Assumption}

\Crefname{assumption}{Assumption}{Assumptions}

\makeatletter
\tcbset{before upper={\def\verbatim@font{\normalfont\ttfamily\fontsize{7.5}{9}\selectfont}}}
\makeatother

\title{Compatibility Matters: A Framework for Token-Level Teacher Selection in On-Policy Distillation}

\usepackage{authblk}

\author[1]{Stanley Wei\thanks{Work done during internship at Amazon Annapurna Labs.}}
\author[2]{Jiin Woo}
\author[2]{Kaan Ozkara}
\author[2]{Youngsuk Park}
\affil[1]{Princeton University}
\affil[2]{Amazon Annapurna Labs}
\affil[ ]{\small\texttt{stanley.wei@princeton.edu}, \texttt{\{woojiin, kaanozka, pyoungsu\}@amazon.com}}

\date{} 

\begin{document}
\maketitle

\begin{abstract}
On-policy distillation (OPD) has recently been introduced as a promising paradigm for continual learning, where a student policy can be supervised at the token-level. Typically, OPD is performed using a single teacher per prompt rollout (or even per training run). However, this ignores a whole other axis of flexibility that is available due to the on-policy nature — \textit{token-level teacher selection} during rollouts. Building on this insight, we theoretically explore what makes a teacher and a student compatible for OPD, and leverage our findings to derive several empirical heuristics for teacher selection. In particular, our framework captures several desiderata for OPD: compatible teacher-student reasoning patterns, teacher competence, and induced student diversity. Theoretically, we categorize several types of teacher routing schemes, and we prove that under different scenarios, different routing schemes are likely to be beneficial for optimization and student diversity. We back up our findings with extensive experimental results in domains such as math, coding, kernel generation, and visual reasoning, where we find upwards of 10 percentage points of improvement on pass@$k$ when compared to vanilla OPD and the teacher experts on certain tasks, indicating both capability improvement and diversity over high-quality responses.
\end{abstract}
\section{Introduction}

Modern LLM training generally follows the framework of a large-scale pretraining phase on corpora from the real world \citep{brown2020languagemodelsfewshotlearners, olmo2026olmo3}, such as internet and human-curated data, followed by a post-training phase, in which a base pretrained LLM is fine-tuned to elicit domain-specific gains in capability, most notably in domains such as mathematical reasoning \citep{shao2024deepseekmathpushinglimitsmathematical}, instruction following \citep{ouyang2022instruct}, and software engineering \citep{wei2025swerladvancingllmreasoning}. In particular, reinforcement learning (RL) \citep{christiano2023deepreinforcementlearninghuman, deepseek2025r1} has recently emerged as a popular paradigm for improving LLM capabilities during post-training. This approach has been well-demonstrated to bring about scalable gains over the initial pre-training phase of LLMs, and has been widely adopted amongst both open- and closed-source models. 
On-policy distillation \citep{hinton2015distilling,agarwal2024gkd} has been proposed as a promising alternative. The difference, however, is that the teacher provides dense, token-level feedback to the student rather than in RL where the single reward scalar signal guides the gradient updates. Recently, specialized versions of OPD such as on-policy self-distillation (OPSD) as well as other variants have been proposed \citep{zhao2026opsd,li2026unsupervisedopsd}. Separately, there have also been works inspired from classical RL literature which train a separate critic model to give an estimate of the process reward \citep{bai2022traininghelpfulharmlessassistant, venkatraman2026lecritiqueprivilegedvalue}.

In this paper, we aim to give a rigorous characterization of what makes a teacher and student compatible for OPD-style training. Informally speaking, the student and teacher should not have too much distributional mismatch \citep{fu2026revisiting, li2026rethinkingopd}; otherwise, as we prove, the optimization landscape will be very unfavorable. Moreover, the teacher should be somewhat competent, which we will quantify later. Finally, the teacher's supervision should induce good diversity over good responses for the student. We formalize all of these concepts in a unified manner in \Cref{sec:theory}. Importantly, we make an additional observation that is unique to the on-policy setting: not only can we do teacher selection for a given prompt (i.e. as in MOPD \citep{ma2026mopdmultiteacheronpolicydistillation} where prompts of a given domain are routed to the corresponding domain expert), we can actually also do teacher selection \textit{token-wise} during a student rollout. The intuition for this is that different teachers may be more suitable for different parts of a student rollout (for instance, in algorithmic coding it may be useful to have a math teacher guide the mathematical reasoning prior to implementation).

Overall, we summarize our contributions as follows. 
\begin{itemize}[itemsep=1pt, parsep=0pt, topsep=0pt, partopsep=0pt]
    \item First, we give a theoretical framework to ground the analysis of token-level teacher selection, in which we can decouple the contribution of teacher competence and of the teacher-student gap.
    \item Then, within this framework, we give a unified mathematical interpretation of the confidence-based routing scheme from \citet{yin2026hopdconfidenceawareheterogeneous} and a \textit{newly proposed representation-based} routing scheme. In particular, we show in different scenarios one or the other could prove more useful as a routing scheme. Importantly, we also prove that the representation-based routing has the potential to yield more diverse student generations.
    \item We back up our theoretical insights with experiments across many domains, including code, Triton kernel generation, math, and visual reasoning. These demonstrate the clear potential of token-level teacher selection over vanilla OPD via improvements across pass@$k$, indicating student learning of high quality and diverse responses. More importantly, we are able to have principled predictions of when a family of routing schemes is likely to be effective in practical setups.
\end{itemize}

We discuss related work in \Cref{app:related-work}.

\section{Preliminaries}
\label{sec:framework}

\subsection{On-Policy Distillation Formulation}
Consider the setup where we have a student policy $\pi_\theta$ parameterized by $\theta \in \mathbb R^d$. Let $x\sim\mathcal D$, where $\mathcal{D}$ is some prompt distribution, $y=(y_1,\ldots,y_{|y|}) \sim \pi_\theta(\cdot | x)$ be a generation of a student response, and $s_t=(x,y_{<t})$ be the state that predicts the next token $y_t$. Then, vanilla OPD minimizes the following loss:
\begin{align*}
    \mathcal{L}_\text{OPD}(\theta) &:= \E_{x\sim\mathcal{D}} \qty[\mathrm{KL}(\pi_\theta \| \pi_T)] \\
    &= \E_{x\sim\mathcal{D}, y\sim\pi_\theta(\cdot|x)} \qty[\sum_{t=1}^{|y|} \mathrm{KL}(\pi_\theta(\cdot| s_t) \| \pi_T(\cdot|s_t))],
\end{align*}
where $\pi_T$ defines some teacher policy. In particular, one can calculate the gradient of this as:
\begin{align*}
    &\nabla \mathcal{L}_\text{OPD}(\theta) \nonumber\\
    &= \E_{x\sim\mathcal{D}, y\sim\pi_\theta(\cdot|x)} \qty[\sum_{t=1}^{|y|} \sum_{t'=t}^{|y|} \log\frac{\pi_\theta(y_{t'}|s_{t'})}{\pi_T(y_{t'}|s_{t'})} \nabla_\theta \log \pi_\theta(y_t|s_t)].
\end{align*}

However, in practice, this gradient is approximated as:
\begin{align}
     \E_{x\sim\mathcal{D}, y\sim\pi_\theta(\cdot|x)} \qty[\sum_{t=1}^{|y|}  \log\frac{\pi_\theta(y_{t}|s_{t})}{\pi_T(y_{t}|s_{t})} \nabla_\theta \log \pi_\theta(y_t|s_t)].
    \label{eq:opd-immediate-direction}
\end{align}
Observe that this approximation is equivalent to moving the gradient inside of the expectation in the original loss (i.e. $\E_{x\sim\mathcal{D}, y\sim\pi_\theta(\cdot|x)} \qty[\sum_{t=1}^{|y|} \nabla \mathrm{KL}(\pi_\theta(\cdot| s_t) \| \pi_T(\cdot|s_t))]$ instead of $\nabla \E_{x\sim\mathcal{D}, y\sim\pi_\theta(\cdot|x)} \qty[\sum_{t=1}^{|y|} \mathrm{KL}(\pi_\theta(\cdot| s_t) \| \pi_T(\cdot|s_t))]$).
Our token-level teacher routing acts directly on the teacher term
in \Cref{eq:opd-immediate-direction}. At each sampled prefix $s_t$,
we plug in our choice of $\pi_T(\cdot\mid s_t)$: either a single
teacher selected by hard routing or an arithmetic mixture of teachers
under soft routing, as defined in \Cref{sec:selection}.
This choice determines the sampled-token advantage; the teacher
selection and mixture weights are held fixed during the update.

Note that this approximation is equivalent to using $\log\pi_T(y_t|s_t) - \log\pi_\theta(y_t|s_t)$ as the token-level advantage in OPD. In the rest of our exposition, we will refer to the vanilla OPD above as the \textit{sampled-token advantage} OPD. This is because we will also consider a variant of the above, which is the Top-K KL loss version of OPD; we provide the precise formulation of this in \Cref{app:training-objectives}.

\section{Token-Level Teacher Selection}
\label{sec:selection}

Given $C$ teachers $\pi_{T_1},\ldots,\pi_{T_C}$\footnote{Throughout our paper, we will consider for both our theory and experiments the setting of $C=2$.}, we would like to understand which of these teachers best supervises the student rollout at each token. In our context, \textit{token-level supervision} simply means choosing a teacher for each prefix state $s_t$ for the gradient update in \Cref{eq:opd-immediate-direction}. Regardless of the way we define success for our training run, we already know a priori that there exists a token-level teacher routing scheme that cannot do worse than vanilla OPD (i.e. fixing a teacher throughout training), as we now have combinatorially more options across the entire axis of rollout horizon; we formalize this fact in \Cref{sec:main-performance-diff-lemma} and \Cref{app:reward-decomposition}, where the notion of success is suboptimality of reward with respect to some golden reference distribution.

We now describe the general framework used for routing, first introduced in \citet{yin2026hopdconfidenceawareheterogeneous}. At iteration $j$ of training, generate responses from $\pi_{\theta_j}$ and compute scores
$S_i(s_t;\theta_j)$ for teacher $i$. For a temperature $\tau>0$\footnote{Each criterion can be used for hard teacher selection ($\tau=0$) or a soft
mixture.}, define the weights
\begin{equation}
 w_i(s_t)=
 \frac{\exp(S_i(s_t;\theta_j)/\tau)}
      {\sum_{h=1}^C\exp(S_h(s_t;\theta_j)/\tau)}.
 \label{eq:router}
\end{equation}
Then, we use the mixture teacher $\pi_{T,w}=\sum_i w_i\pi_{T_i}$ as the supervision target of prefix $s_t$. Using this framework, we now give three classes of routing schemes.

\textbf{Direct section routing.} When the required skill is clear from
the rollout section, we can directly assign its domain teacher. That is, for a student rollout $y=(y_1,\dots, y_H)$, if token $t$ corresponds to the domain of teacher $i$, one can consider directly routing the teacher at token $t$ to be teacher $i$ by taking $\tau=0$ and $S_i(s_t;\theta_j) > S_{i'}(s_t;\theta_j)$ for $i'\neq i$ in \Cref{eq:router}, from which the gradient in \Cref{eq:opd-immediate-direction} will be calculated. For
example, in the setting of writing kernel code, a thinking teacher can supervise a reasoning section, while
a kernel specialist can supervise the subsequent implementation.

\textbf{Teacher confidence-based routing.} This was first introduced in \citet{yin2026hopdconfidenceawareheterogeneous} through their H-OPD algorithm. In this family of routing scores, $S_i(s_t;\theta_j)$ is taken to be the negative entropy of the teacher at that prefix; that is, $S_i(s_t;\theta_j) = -H(\pi_{T_i}(\cdot | s_t))$. Note that this is independent of the \textit{student} policy at $s_t$ and is calculated purely from the teacher's confidence. In some sense, this can be seen as a proxy estimating which teacher yields more successful rewards on its continuation from a student prefix; this proxy is more accurate of course when the teacher is correctly confident.

\textbf{Teacher-student representation-based routing.}
We introduce a new, more general family of routing schemes that captures the state of the student policy at $s_t$, by considering the teacher compatibility with the student in the sense of match in representation space or similarity of thinking patterns. Empirically, we implement the following scores\footnote{Our empirical implementations only consider these scores, but we nevertheless believe it would be interesting future work to explore other score functions of the same spirit of representation/reasoning-pattern match.}: 1) the negative Hellinger distance\footnote{For distributions $\mu,\nu$
over a set $\Omega_t$, the Hellinger distance between the two distributions is defined as $H(\mu,\nu):=\left[\frac12\sum_{v\in\Omega_t}
       \bigl(\sqrt{\mu(v)}-\sqrt{\nu(v)}\bigr)^2\right]^{1/2}.$}
between the student and teacher policy at $s_t$, and 2) the negative squared Euclidean distance between the teacher and student features at $s_t$ (that is, the final layer before the unembedding layer, denoted as $\phi(s_t)$ on prefix $s_t$ for the student, and $\psi_i(s_t)$ for teacher $i$). We remark that feature matching is especially natural for shared
architectures, such as a teacher post-trained from the student; otherwise, we resort to the Hellinger distance.

\section{Theoretical Analysis of Teacher Routing Schemes}
\label{sec:theory}
In this section, we provide a theoretical framework to unify the different families of routing schemes described in \Cref{sec:selection}, in an aim to understand when these schemes may be useful.

\subsection{Decoupling Teacher Competence and the Student-Teacher Gap}
\label{sec:main-performance-diff-lemma}
To begin, we introduce the performance difference lemma \citep{kakade2002approximately,schulman2015trpo}, which allows us to roughly decouple the teacher competence (as measured via the reward suboptimality to some target distribution) and the student optimization toward the teacher.
Let $R(x,y)\in[0,1]$
be the relevant task reward, and define the performance of any
autoregressive policy $\pi$ by 
\[
J(\pi):=\mathbb E_{x\sim\mathcal D,\,y\sim\pi(\cdot\mid x)}R(x,y).
\]
Let $\pi^\star$ denote a fixed golden reference policy, $\pi_\theta$
the student, and $\pi_T$ a fixed teacher policy. 
For a policy $\pi$ and prefix $s_t=(x,y_{<t})$, let $V_t^\pi(s_t)$
be the expected terminal reward when all remaining tokens are drawn
from $\pi$, and let $Q_t^\pi(s_t,v)$ be the expected terminal reward when
the next token is fixed to $v$ and subsequent tokens follow $\pi$.
Observe that $V_t^\pi(s_t)=\sum_v\pi(v\mid s_t)Q_t^\pi(s_t,v)$, and define
the task-reward advantage
$A_t^\pi(s_t,v):=Q_t^\pi(s_t,v)-V_t^\pi(s_t)$.
Finally, we denote $d_t^{\pi_\theta}$ as the distribution of $s_t$ obtained
by sampling $x\sim\mathcal D$ and the first $t-1$ student tokens.

\paragraph{Performance difference lemma.}
Let the teacher's gap to the golden reference policy $\pi^\star$ be
$\Delta_T:=J(\pi^\star)-J(\pi_T)$. Then, it holds that:
\begin{align*}
 &J(\pi^\star)-J(\pi_\theta)\\
 &=\underbrace{\Delta_T}_{\text{teacher competence gap}}
   +\underbrace{J(\pi_T)-J(\pi_\theta)}_{\text{student-teacher gap}}
   \nonumber\\
 &=\Delta_T-\sum_{t=1}^H
    \mathbb E_{s_t\sim d_t^{\pi_\theta},\,v\sim\pi_\theta(\cdot\mid s_t)}
          A_t^{\pi_T}(s_t,v)\nonumber.
\end{align*}

Under this setup, we now illustrate the potential of token-wise teacher selection, using the following intuitive result regarding the sum of the token-level advantages induced by the teacher on the right hand side of the above equality.

\begin{lemma}\label{lem:tokenwise-is-better}
Suppose we have access to two teachers $\pi_{T_1}, \pi_{T_2}$, and we have a routing scheme $p(\cdot)$ that maps prefix states $s_t$ to $\{T_1, T_2\}$. Then, it holds that:
\begin{align*}
    & \min_{p} \left\{ -\sum_{t=1}^H\mathbb E_{s_t\sim d_t^{\pi_\theta}, v\sim\pi_\theta(\cdot|s_t)} A_t^{\pi_{p(s_t)} } (s_t, v)  \right\} \\
             \leq& -\sum_{t=1}^H\mathbb E_{s_t\sim d_t^{\pi_\theta}, v\sim\pi_\theta(\cdot|s_t)} A_t^{\pi_{T_i}}  (s_t, v) \quad \text{for } i=1,2
\end{align*}
\end{lemma}
\begin{proof}
    This follows from the fact that the right hand side is just the routing scheme $p(\cdot) = T_1$ or $p(\cdot)=T_2$.
\end{proof}

This token-wise decomposition suggests an entire underexplored axis of teacher selection, and it is perfectly plausible that something as naive as section routing should yield better performance than a fixed teacher (nevertheless, we should generally hope for something more principled to exist). Henceforth, we will try to understand how we should perform this selection at the token-level. To begin, we will make the connection between the above student-teacher gap and the reverse KL loss through the following lemma, whose full proof is in \Cref{app:reward-decomposition}.

\begin{lemma}\label{lem:connection-to-KL}
For each prefix $s_t$, define \[
 L_{T,t}(\theta;s_t):=
 \mathrm{KL}\!\left(\pi_\theta(\cdot\mid s_t)
                         \,\|\,\pi_T(\cdot\mid s_t)\right).
\]    
Then, it holds that \begin{equation}
 |J(\pi_T)-J(\pi_\theta)|
 \leq\sum_{t=1}^H\mathbb E_{s_t\sim d_t^{\pi_\theta}}
             \sqrt{\tfrac12 L_{T,t}(\theta;s_t)}.
 \label{eq:reward-local-kl-bound}
\end{equation}
\end{lemma}

\Cref{lem:connection-to-KL} upper bounds the teacher-student imitation gap in terms of a sum of the different token-level KL terms. While this is not the exact gap term, it is a meaningful surrogate that allows tractable theoretical analysis. Furthermore, this is exactly what OPD optimizes for: minimizing the token-level KL between a student and a given teacher. In the next few sections, we will be concerned with understanding how to choose a teacher that minimizes this student-teacher imitation gap from quantifying the student-teacher compatibility, and using that to demonstrate a separation from the optimization perspective.

\subsection{Unifying Representation- and Entropy-Based Routing}
\label{sec:unified-routing}

\paragraph{Setup.} Let $\mathcal X$ denote the set of prompts, $\mathcal{V}$ denote the vocabulary, and $\mathcal{V}^\star$ denote the set of finite sequences over $\mathcal{V}$. We work in the log-linear model family, which is defined by a fixed feature mapping $\phi : \mathcal{X} \times \mathcal{V}^\star \rightarrow \mathbb{R}^d$. Such a model family is parameterized by $\theta\in\mathbb{R}^d$, and for a given rollout of length $H$, we define $(\pi_\theta)_{t=1}^H$ by:
\begin{align*}
    \pi_\theta(y_t | s_t) \propto \exp\qty(\langle \theta, \phi(s_t, y_t) \rangle).
\end{align*}
Note that the log-linear setting is an abstraction of real-world autoregressive models such as transformers \citep{vaswani2023attentionneed} with fixed hidden features. In our setting, we will assume that the student has fixed features $\phi(\cdot)$ and trainable linear head $\theta\in\mathbb R^d$, and teacher $i$ has fixed features $\psi_i(\cdot)$ and linear head $\omega_i\in\mathbb R^d$. That is, we have the following:
\begin{equation*}
\begin{aligned}
    \pi_\theta(y_t | s_t) &\propto \exp\qty(\langle \theta, \phi(s_t, y_t) \rangle) \\
    \pi_{T_1}(y_t | s_t) &\propto \exp\qty(\langle \omega_1, \psi_1(s_t, y_t) \rangle) \\
    \pi_{T_2}(y_t | s_t) &\propto \exp\qty(\langle \omega_2, \psi_2(s_t, y_t) \rangle),
\end{aligned}
\end{equation*}

where $\phi(\cdot), \psi_1(\cdot), \psi_2(\cdot) \in \mathbb{R}^d$. Note that $\theta$ is the only trainable parameter here, as we will not be training the teacher. In the exposition that follows, we will hold $s_t=(x,y_{<t})$ fixed, so that we can localize our analysis to the contribution of a single prefix.\footnote{In practice, the parameters indeed affect all the tokens of an autoregressive rollout in an intricate way. However, for simplicity, we assume that different tokens are localized to a different part of the parameter space.} In addition, we will not consider the additive mixture models as we do in our empirical experiments; in other words, our theory analyzes only hard routing ($\tau=0$).

\paragraph{A unified framework.} We briefly lay out some heuristic calculations on when we can expect different routing schemes to be beneficial. First, suppose that $\rho_i$ is some scalar quantity that quantifies the alignment between the teacher and student features\footnote{More precisely, in our setup detailed in \Cref{app:covariance-transfer-proof}, $\rho_i = 1 - \Tilde\Theta(D_i(t))$, where $D_i(t)$ is defined in \Cref{sec:covariance-transfer}.}, and let $\theta_0$ denote the student head at initialization. Under our setup, the initial gradient can be written as:
\begin{equation}
 \begin{aligned}
-\nabla L_{T_i, t}(\theta_0;s_t) &= \underbrace{\rho_i \omega_i}_{\text{accessible teacher signal}}
   -\theta_0, \\
   \implies \|\nabla L_{T_i, t}(\theta_0;s_t)\|_2^2 &= \rho_i^2 \|\omega_i\|_2^2 + \|\theta_0\|_2^2 - 2\rho_i  \langle \omega_i, \theta_0\rangle \\
   &\approx \rho_i^2 \|\omega_i\|_2^2
 \end{aligned}
 \label{eq:unified-routing-main}
\end{equation}

In this simplification, we observe that the main (scalar) quantities affecting the accessible teacher signal are the student-teacher representation alignment coefficient $\rho_i$, as well as the magnitude of the teacher head $\|\omega_i\|_2$. Together, these two factors explain the complementary roles of the routing schemes. For instance, assuming different teacher head magnitudes are equal, choosing the more aligned teacher roughly maximizes the initial gradient norm. On the other hand, assuming that the teacher-student feature alignments are of roughly similar scale, choosing the teacher that has higher norm roughly corresponds to choosing the more confident teacher (and also maximizes the initial gradient norm). In particular, our theoretical analysis demonstrates that the initial gradient norm determines how much movement in parameter space will happen, and hence regulates how low we can get in terms of the reverse KL loss in \Cref{eq:reward-local-kl-bound}.
Of particular note is that as we have decoupled the contribution of teacher competence in the performance difference lemma, the teacher reward does not generally have an explicit contribution in this gradient. We defer a more detailed discussion of this to \Cref{app:alignment-confidence}.

\subsection{Representation-Based Teacher Routing}
\label{sec:covariance-transfer}
In this section, we will show two results about representation-based routing: how it 1) yields more gradient movement, and 2) yields more diversity over the overlapping feature space, which is beneficial when good modes are distributed relatively uniformly.

Our main result here is \Cref{thm:learning-time-abridged}, which states that under mild conditions on the teacher and student features, the teacher that is more aligned with the student induces a larger gradient norm. For any function $f$ of a candidate token, define the shorthand $\mathbb E_t[f(v)]
 :=\sum_{v\in\mathcal V}\pi_{\theta_0}(v\mid s_t)f(v)$
where $\theta_0$ is the student head's parameter at initialization. Consider the contribution of the reverse KL loss at token $t$ of teacher $i$, 
\begin{align*}
 L_{i,t}(\theta)&:=
 \sum_{v\in\mathcal V}\pi_\theta(v\mid s_t)
             \log\frac{\pi_\theta(v\mid s_t)}{\pi_{T_i}(v\mid s_t)}.
\end{align*}

We now connect this initial update to the time needed to reduce the
teacher--student KL. Define the mean squared feature discrepancy
and the initial gradient norm by
\[
 D_i(t):=\mathbb E_t\|\phi(s_t,v)-\psi_i(s_t,v)\|_2^2,
 \qquad a_i:=\|\nabla L_{i,t}(\theta_0)\|_2.
\]
Starting from the same student head $\theta_0$, let $\theta_n^{(i)}$ be the
student head after $n$ gradient updates using teacher $i$:
\[
 \theta_0^{(i)}=\theta_0,\qquad
 \theta_{n+1}^{(i)}
 =\theta_n^{(i)}-\eta\nabla L_{i,t}(\theta_n^{(i)}),
\]
where $\eta>0$ is the learning rate. For a target loss level $0<\varepsilon<\min_i L_{i,t}(\theta_0)$, define the learning time
\[
 N_i(\varepsilon):=
 \inf\{n\in\mathbb Z_{\geq0}:L_{i,t}(\theta_n^{(i)})\leq\varepsilon\},
\]
Then, we have the following separation in learning time.

\begin{theorem}[Learning time guarantee for feature-based routing, abridged]
\label{thm:learning-time-abridged}
Suppose that $\|\omega_1\|_2=\|\omega_2\|_2$. Then, under suitable assumptions, we have that if $D_1(t)<D_2(t)$, then $a_1>a_2$. For a sufficiently small learning rate $\eta$, 
\[
 N_1(\varepsilon)=O\!\left(\frac{1}{\eta a_1^2}\right) \lesssim 
 N_2(\varepsilon)=\Omega\!\left(\frac{1}{\eta a_2^2}\right).
\]
\end{theorem}

\paragraph{Proof sketch.}
Under mild assumptions, we can directly compute the gradients at initialization, the magnitude of which scales with decreasing $D_i(t)$. Using this, as well as smoothness of the loss landscape, we obtain an upper bound on $N_1(\epsilon)$ and a lower bound on $N_2(\epsilon)$, which yields the separation.
The full unabridged version of this theorem statement is in \Cref{app:learning-time-proof}.

\subsection{Confidence-Based Routing}
\label{sec:entropy-routing-theory}

We now consider the complementary setting where the two teachers share a feature map $\psi$, but with heads of differing magnitudes. It turns out that in this case, teacher confidence is correlated with the magnitude of its head, which consequently controls the magnitude of the gradient signal to the student. Fix the prefix $s_t$
and initial student head $\theta_0$ as before, and define the
teacher's entropy by
\[
 H_i:=-\sum_{v\in\mathcal V}\pi_{T_i}(v\mid s_t)
                       \log\pi_{T_i}(v\mid s_t).
\]

\begin{theorem}[Learning time guarantee for entropy-based routing, abridged]
\label{thm:entropy-gradient-abridged}
Suppose $\psi(s_t,v):=\psi_1(s_t,v)=\psi_2(s_t,v)$ for every
token $v$; that is, the two teachers share the fixed features. Then, under suitable assumptions, we have that if $H_1 < H_2$, then it holds that $a_1 > a_2$. For a sufficiently small learning rate $\eta$, it holds that 
\[
 N_1(\varepsilon)=O\!\left(\frac{1}{\eta a_1^2}\right) \lesssim N_2(\varepsilon)=\Omega\!\left(\frac{1}{\eta a_2^2}\right).
\]

\end{theorem}

\paragraph{Proof sketch.}
Roughly speaking, the scale of teacher logits increases with the magnitude of $\|\omega_i\|_2$. Therefore, under mild assumptions, we show that increasing $\|\omega_i\|_2$ increases the logit gaps between different tokens, thus strictly lowering the entropy. As such, teacher 1 will have a stronger pull, in the sense of a large gradient, and from there, we can conclude the proof in a similar manner as \Cref{thm:learning-time-abridged} to obtain the separation.
The full proof is deferred to \Cref{app:confidence-routing}.

\subsection{Overlap with Good Solutions and Coverage}
\label{sec:main-coverage}

There is an additional benefit of our representation-based routing scheme from \Cref{sec:covariance-transfer} worth highlighting. Assuming high-quality modes are uniformly distributed, choosing a more aligned teacher gives a stronger guarantee on how many of these modes the optimized student retains. Put simply, we have stronger guarantees on the diversity over good responses learned by the student. We present the abridged theorem here.

\begin{theorem}[Diversity guarantee, abridged]
\label{thm:coverage-abridged}
Let $G_1,\ldots,G_M$ be disjoint, nonempty next-token sets identifying
$M\geq1$ successful continuation modes.
Assume a fixed shared continuation mechanism: tokens in $G_j$ lead
to success in mode $j$, and tokens outside their union fail; moreover, suppose that both teachers place probability $c/M$ on each mode, for some common success probability $c\in(0,1)$. Then, there exists a
teacher-dependent error term
$e_i\in[0,c/M)$ that increases as $D_i(t)$ increases, which satisfies that the expected number of distinct
successful modes observed in $k$ independent samples from $\pi_{\theta_i^\star}$ is at least $M\bigl[1-(1-c/M+e_i)^k\bigr]$, where $\theta_i^\star = \argmin_\theta L_{i,t}(\theta)$.
\end{theorem}

\paragraph{Proof sketch.}
The feature overlap provides the direction for learnable teacher signal. Under the uniform mode assumption, we can prove that the student, when optimized against teacher $i$, will retain at least $c/M-e_i$ probability on every mode, yielding the desired result. 

Furthermore, a related concept is the coverage profile perspective of \citet{chen2025coverage}, which characterizes how well a distribution covers some target distribution (and is also predictive of pass@$k$ performance).
For our setting, let $\nu(j):=1/M$ be the target distribution,
uniform over the $M$ good modes. For a threshold parameter $N>1/c$, we adapt the framework in \citet{chen2025coverage} to define the \textit{undercoverage} by
\begin{equation}
 C_N^{\rm mode}(\nu\|\pi_{\theta_i^\star})
 :=\frac1M\sum_{j=1}^M
    \mathbf1\!\left\{\pi_{\theta_i^\star}(G_j)\leq\frac1{MN}\right\}.
 \label{eq:mode-undercoverage-definition}
\end{equation}

\begin{corollary}[Undercoverage of successful modes]
\label{cor:feature-undercoverage}
In the setting of \Cref{thm:coverage-abridged}, for every $N>1/c$,
\begin{equation*}
 C_N^{\rm mode}(\nu\|\pi_{\theta_i^\star})
 \leq \frac{e_i}{c-1/N}.
\end{equation*}
\end{corollary}

We remark that our analysis here provides a comparison of the lower bounds of expected number of distinct modes, rather than the value itself. Nevertheless, we believe it is still a reasonable surrogate to look at; our experiments demonstrate that our representation-matching method indeed yields higher pass@$k$ even when pass@$1$ is not necessarily higher. We give the full proof of all results above in \Cref{app:coverage-statements}.

\section{Experimental Results}
\label{sec:experiments}

\begin{figure*}[!t]
\centering
\includegraphics[width=\textwidth]{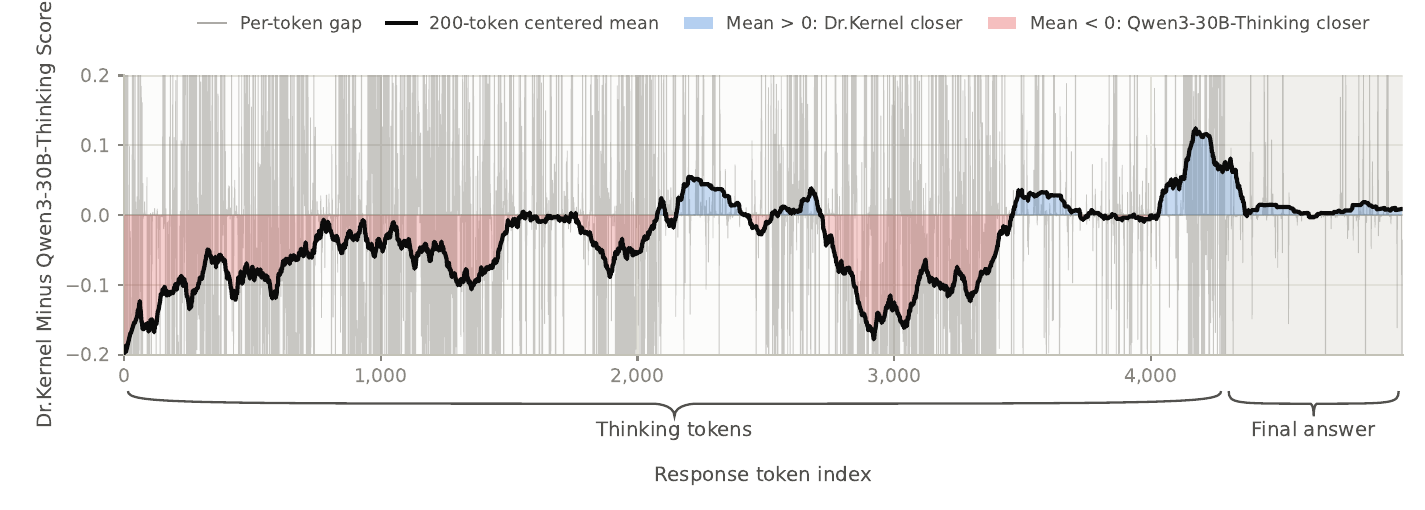}
\caption{Triton-only entropy routing ($\tau=0.1$; iteration 300). Here, we plot the score gap of the kernel teacher minus the thinking teacher. The full trace is included in \Cref{app:figure-rollouts-triton}, where we can observe that the blue (positive) regions on this plot correspond roughly to parts of the student rollout where Triton code is being generated, and the red parts correspond to general reasoning or thinking behavior to plan the kernel code.}
\label{fig:thinking-routing-probe}
\end{figure*}

We study token-level teacher selection in settings that encompass the domains of math, coding, Triton kernel generation, and visual reasoning\footnote{Results on visual reasoning are deferred to \Cref{app:figure-rollouts}.}. More specifically, we consider the following setups:
\begin{itemize}[itemsep=1pt, parsep=0pt, topsep=1pt, partopsep=0pt]
      \item A synthetic math-triton task with access to respective math and Triton teachers where there is a clear way to perform teacher selection via section routing (\Cref{sec:synthetic-experiments})
      \item A Triton-only task with access to a thinking model and a domain expert (\Cref{sec:thinking})
      \item Algorithmic coding tasks under a math and code teacher (\Cref{sec:algo-coding})
\end{itemize}

Precise scoring functions used are defined in \Cref{sec:score-definitions}, as they may differ in implementation from our theoretical formulations in \Cref{sec:theory}. One concrete example of this is when we choose to use the Hellinger score instead of the feature distance score in some cases, due to the teacher and student having different architectures.
Additional training details and evaluation protocols are given in
\Cref{app:hyperparameters,app:eval-settings}. Notably, for our score-based routing frameworks, we will generally sweep through several temperatures $\tau$, as well as choices of policy gradient implementation (i.e. either sampled token advantage or the Top-K KL). For sampled token advantage, we use $\tau=0$ (i.e. hard routing) by default. In our tables, besides reporting results on token-level routing, we also report baselines of vanilla OPD (which uses sampled token advantage by default), as well as the teachers and initial student performance.

\vspace{-0.6\baselineskip}
\subsection{A Synthetic Math-Triton Task}
\label{sec:synthetic-experiments}

We construct more than 8000 training prompts by pairing DeepMath problems
\citep{he2025deepmath} with KernelBook tasks \citep{kernelbook2025}. Each prompt requests both solutions, with
the task order assigned randomly and held fixed. The student model we use is a Qwen3-4B \citep{qwen3techreport}, using the non-base variant with thinking mode off. For the teachers, we use a Qwen3-4B that has gone through math RL training on DeepMath, as well as the Dr.Kernel-8B \citep{liu2026drkernel}. We train the student for 300 updates on these composite prompts.

For evaluation, we use the size 3000 set of 
AIME 2025 $\times$ KernelBench Level 2 \citep{aime2025opencompass, ouyang2025kernelbench} combinations (30 AIME problems and 100 KernelBench Level 2 problems). Joint success
requires a correct math answer and a correct kernel in the same
response, with kernel correctness measured by fast$_0$, which gives a binary reward purely depending on correctness, without regard to speedup. Our results are shown in \Cref{tab:synthetic-l2}, in which we compare against the baselines of the initial student model, the teachers, and vanilla OPD. With token-level routing, our Hellinger routing metric performs the best on this joint fast$_0$ score, for both pass@1 and pass@16, reaching 29.30\% on pass@16, which is close to 10 percentage points over vanilla OPD on either teacher, the entropy-routing training, and even the naive section-routing variant.

\begin{table}[!t]
\centering
\small
\setlength{\tabcolsep}{3pt}
\caption{AIME 2025 $\times$ KernelBench L2 (\%). Joint fast$_0$ requires correctness of both components in one response (higher is better); we also report the math-portion and kernel-portion breakdown for completeness. Note that for the sampled advantage Hellinger routing, the scores on each individual section recover roughly the teacher performance, and together yield the highest joint score.}
\label{tab:synthetic-l2}
\begin{tabular}{lrr|rr}
\toprule
& \multicolumn{2}{c}{Joint fast$_0$} & Math-portion & Kernel-portion \\
\cmidrule(lr){2-3}\cmidrule(lr){4-4}\cmidrule(lr){5-5}
Method & pass@1 & pass@16 & pass@1 & pass@1 \\
\midrule
\multicolumn{3}{l|}{\textit{Initial student model}} & & \\
Qwen3-4B & 1.79 & 15.04 & 15.08 & 12.04 \\
\midrule
\multicolumn{3}{l|}{\textit{Teachers}} & & \\
Math RL expert & 2.87 & 19.92 & 23.85 & 11.46 \\
Dr.Kernel-8B & 1.42 & 14.06 & 6.74 & 17.57 \\
\midrule
\multicolumn{3}{l|}{\textit{Without token-level routing}} & & \\
OPD (math teacher) & 2.84 & 20.12 & 24.06 & 9.69 \\
OPD (Dr.Kernel teacher) & 1.59 & 13.09 & 6.40 & 22.67 \\
\specialrule{\lightrulewidth}{\aboverulesep}{0pt}
\rowcolor{routingrowcolor} \multicolumn{3}{l|}{\routingtopgap\textit{With token-level routing}} & & \\
\rowcolor{routingrowcolor} Section-routed (sampled advantage) & 2.66 & 20.90 & 13.38 & 20.20 \\
\rowcolor{routingrowcolor} Entropy + TopK-KL ($\tau=0.1$) & 2.38 & 19.14 & 17.66 & 12.34 \\
\rowcolor{routingrowcolor} Hellinger (sampled advantage) & \textbf{4.37} & \textbf{29.30} & 24.26 & 17.09 \\
\rowcolor{routingrowcolor} Hellinger + TopK-KL ($\tau=0.01$)\routingbotgap & 4.17 & 28.52 & 23.94 & 16.54 \\
\specialrule{\heavyrulewidth}{0pt}{\belowbottomsep}
\end{tabular}
\end{table}

\subsection{Triton-only Task with Thinking Mode}
\label{sec:thinking}

\begin{figure}[!t]
\centering
\colorlet{mathred}{red!80!black}\colorlet{codeblue}{blue!80!black}
\begin{tikzpicture}
\node[draw=black!40, line width=0.5pt, rounded corners=5pt, inner xsep=6pt, inner ysep=5pt, fill=black!2] {\begin{minipage}{0.93\columnwidth}
\footnotesize\setlength{\parskip}{1pt}\setlength{\parindent}{0pt}\raggedright
\textcolor{codeblue!30!mathred}{To} \textcolor{codeblue!40!mathred}{solve the} \textcolor{codeblue!30!mathred}{largest rectangular area} \textcolor{codeblue!35!mathred}{in} \textcolor{codeblue!40!mathred}{a} \textcolor{codeblue!45!mathred}{histogram} \textcolor{codeblue!50!mathred}{problem} \textcolor{codeblue!55!mathred}{efficiently in O}\textcolor{codeblue!60!mathred}{(N)} \textcolor{codeblue!50!mathred}{time and} \textcolor{codeblue!40!mathred}{O(N) space}\textcolor{codeblue!45!mathred}{,} \textcolor{codeblue!50!mathred}{we} \textcolor{codeblue!60!mathred}{can use} \textcolor{codeblue!55!mathred}{a} \textcolor{codeblue!30!mathred}{mon}\textcolor{codeblue!20!mathred}{ot}\textcolor{codeblue!25!mathred}{onic stack} \textcolor{codeblue!20!mathred}{approach}\textcolor{codeblue!25!mathred}{.}\par \par \textcolor{codeblue!10!mathred}{\textbf{Int}}\textcolor{codeblue!0!mathred}{\textbf{uition}\textbf{:}}\par \par \textcolor{codeblue!0!mathred}{For each bar, we want to} \textcolor{codeblue!5!mathred}{find}\textcolor{codeblue!15!mathred}{:}\par \textcolor{codeblue!30!mathred}{-} \textcolor{codeblue!25!mathred}{The} \textcolor{codeblue!30!mathred}{first} \textcolor{codeblue!40!mathred}{bar} \textcolor{codeblue!50!mathred}{to} \textcolor{codeblue!65!mathred}{the} \textcolor{codeblue!70!mathred}{left that} \textcolor{codeblue!85!mathred}{is} \textcolor{codeblue!85!mathred}{shorter than the} \textcolor{codeblue!80!mathred}{current} \textcolor{codeblue!85!mathred}{bar (left} \textcolor{codeblue!80!mathred}{bound)}\par \textcolor{codeblue!80!mathred}{- The} \textcolor{codeblue!75!mathred}{first} \textcolor{codeblue!80!mathred}{bar} \textcolor{codeblue!85!mathred}{to} \textcolor{codeblue!80!mathred}{the right} \textcolor{codeblue!75!mathred}{that is} \textcolor{codeblue!75!mathred}{shorter} \textcolor{codeblue!70!mathred}{than} \textcolor{codeblue!65!mathred}{the} \textcolor{codeblue!55!mathred}{current} \textcolor{codeblue!50!mathred}{bar (}\textcolor{codeblue!45!mathred}{right bound}\textcolor{codeblue!40!mathred}{)}\par \noindent\textcolor{gray}{[\ldots]}\par
\textcolor{codeblue!5!mathred}{\textbf{Python}}\textcolor{codeblue!10!mathred}{\textbf{ Code}}\textcolor{codeblue!20!mathred}{\textbf{:}}\par {\ttfamily\scriptsize\setlength{\parskip}{0pt}\textcolor{codeblue!35!mathred}{def}\textcolor{codeblue!45!mathred}{\ getMax}\textcolor{codeblue!55!mathred}{Area}\textcolor{codeblue!60!mathred}{(arr}\textcolor{codeblue!70!mathred}{):}\par\noindent \textcolor{codeblue!70!mathred}{\ \ \ \ stack}\textcolor{codeblue!80!mathred}{\ =}\textcolor{codeblue!90!mathred}{\ []}\par\noindent \textcolor{codeblue!90!mathred}{\ \ \ }\textcolor{codeblue!85!mathred}{\ max\_area\ =\ }\textcolor{codeblue!80!mathred}{0}\par\noindent \textcolor{codeblue!75!mathred}{\ \ \ \ n\ =\ len}\textcolor{codeblue!80!mathred}{(arr)}\par\noindent \textcolor{codeblue!75!mathred}{\ \ \ \ }\par\noindent \textcolor{codeblue!75!mathred}{\ \ \ \ for\ i\ in}\textcolor{codeblue!80!mathred}{\ range}\textcolor{codeblue!75!mathred}{(n}\textcolor{codeblue!85!mathred}{):}\par\noindent \textcolor{codeblue!90!mathred}{\ \ \ \ \ \ \ \ while\ stack\ and\ arr[i}\textcolor{codeblue!95!mathred}{]\ \textless{}\ arr[stack[-1]]:}\par\noindent \textcolor{codeblue!95!mathred}{\ \ \ \ \ \ \ \ \ \ \ \ height\ =\ arr}\textcolor{codeblue!90!mathred}{[}\textcolor{codeblue!95!mathred}{stack.pop()]}\par\noindent \textcolor{codeblue!95!mathred}{\ \ \ \ \ \ \ \ \ \ \ \ width\ =\ i\ if\ not\ stack\ else\ i}\textcolor{codeblue!100!mathred}{\ -}\textcolor{codeblue!95!mathred}{\ stack}\textcolor{codeblue!100!mathred}{[-1]}\textcolor{codeblue!95!mathred}{\ -\ 1}\par\noindent \textcolor{codeblue!90!mathred}{\ \ \ \ \ \ \ \ \ \ \ \ max}\textcolor{codeblue!85!mathred}{\_area\ =\ max}\textcolor{codeblue!75!mathred}{(max\_area,\ height\ *}\textcolor{codeblue!70!mathred}{\ width}\textcolor{codeblue!75!mathred}{)}\par\noindent \textcolor{codeblue!70!mathred}{\ \ \ \ \ \ \ \ stack.append}\textcolor{codeblue!80!mathred}{(i}\textcolor{codeblue!75!mathred}{)}\par\noindent \par}\par\noindent\textcolor{gray}{[\ldots]}\par
{\ttfamily\scriptsize\setlength{\parskip}{0pt}\textcolor{codeblue!20!mathred}{\ return}\textcolor{codeblue!25!mathred}{\ max}\textcolor{codeblue!20!mathred}{\_area}\par\noindent \par}\textcolor{codeblue!0!mathred}{Test Cases:}\par {\ttfamily\scriptsize\setlength{\parskip}{0pt}\textcolor{codeblue!0!mathred}{\#\ Example\ }\textcolor{codeblue!5!mathred}{1}\par\noindent \textcolor{codeblue!5!mathred}{print}\textcolor{codeblue!10!mathred}{(get}\textcolor{codeblue!15!mathred}{Max}\textcolor{codeblue!25!mathred}{Area}\textcolor{codeblue!35!mathred}{([}\textcolor{codeblue!40!mathred}{6}\textcolor{codeblue!55!mathred}{,}\textcolor{codeblue!60!mathred}{\ }\textcolor{codeblue!65!mathred}{2}\textcolor{codeblue!60!mathred}{,\ }\textcolor{codeblue!55!mathred}{5,\ 4}\textcolor{codeblue!60!mathred}{,\ 5,}\textcolor{codeblue!50!mathred}{\ 1}\textcolor{codeblue!45!mathred}{,}\textcolor{codeblue!40!mathred}{\ 6]))}\textcolor{codeblue!30!mathred}{\ }\textcolor{codeblue!25!mathred}{\ \#\ Output}\textcolor{codeblue!20!mathred}{:}\textcolor{codeblue!25!mathred}{\ }\textcolor{codeblue!20!mathred}{1}\textcolor{codeblue!15!mathred}{2}\par\noindent \par\noindent \par}\par\noindent\textcolor{gray}{[\ldots]}\par

\end{minipage}};
\end{tikzpicture}\\[3pt]
{\footnotesize favors math teacher\;\begin{tikzpicture}[baseline=-0.5ex]\shade[left color=mathred,right color=codeblue] (0,0) rectangle (2.4,0.22);\end{tikzpicture}\;favors code teacher}
\caption{Abridged rollout on algorithmic coding task by a student trained with feature routing, at step 300. Here, the text is colored by the code-minus-math routing score. Note that the coloring demonstrates that the reasoning part of the rollout leans more towards the math teacher, and the Python implementation is markedly closer to the code teacher. Omitted text is marked by $[\ldots]$, and the full per-token score gap and complete rollout are in \Cref{fig:code-routing-probe} and \Cref{app:figure-rollouts-code}.}
\label{fig:code-rollout-colored}
\end{figure}

In this setting, we consider a Qwen3-4B student with thinking mode turned on and train it to generate Triton kernels. For the teachers, we will use the Qwen3-30B-A3B-Thinking model and Dr.Kernel-8B, with the goal of combining general reasoning ability with specialized domain knowledge in Triton. We train the student for 300 steps on KernelBook tasks.

For evaluation, we use the 100 KernelBench Level 2 benchmark problems. In \Cref{tab:thinking-main}, we compare our token-level routing schemes with the base model, the Triton expert, and vanilla OPD. In this thinking model setup, it turns out that entropy-based routing performs the best for fast$_0$ score. In contrast to the previous settings, entropy-based is more effective than the Hellinger-based routing. We illustrate this using an example rollout in \Cref{fig:thinking-routing-probe}.

\begin{table}[!t]
\centering
\small
\caption{Triton-only evaluation on 100 KernelBench Level 2 tasks (\%).
Fast$_0$ measures kernel correctness.}
\label{tab:thinking-main}
\setlength{\tabcolsep}{3pt}
\begin{tabular}{lrr}
\toprule
 & \multicolumn{2}{c}{Fast$_0$} \\
\cmidrule(lr){2-3}
Method & pass@1 & pass@16 \\
\midrule
\multicolumn{3}{l}{\textit{Initial student model}} \\
Qwen3-4B & 9.63 & 53.00 \\
\midrule
\multicolumn{3}{l}{\textit{Teachers}} \\
Dr.Kernel-8B & 12.63 & 78.00 \\
\midrule
\multicolumn{3}{l}{\textit{Without token-level routing}} \\
OPD (Dr.Kernel teacher) & 4.50 & 49.00 \\
OPD (thinking teacher) & 6.12 & 42.00 \\
\specialrule{\lightrulewidth}{\aboverulesep}{0pt}
\rowcolor{routingrowcolor} \multicolumn{3}{l}{\routingtopgap\textit{With token-level routing}} \\
\rowcolor{routingrowcolor} Section-routed (sampled advantage) & 9.12 & 46.00 \\
\rowcolor{routingrowcolor} Entropy (sampled advantage) & \textbf{22.44} & \textbf{82.00} \\
\rowcolor{routingrowcolor} Entropy + TopK-KL ($\tau=0.1$) & 15.56 & 74.00 \\
\rowcolor{routingrowcolor} Hellinger (sampled advantage)\routingbotgap & 14.25 & 69.00 \\
\specialrule{\heavyrulewidth}{0pt}{\belowbottomsep}
\end{tabular}
\end{table}

\subsection{Algorithmic Coding}
\label{sec:algo-coding}

\begin{table}[!t]
\centering
\caption{Algorithmic coding correctness (\%). Although our feature-matching selection does not achieve the highest pass@1, it shows improvements for pass@64, indicating the benefits of diversity of this router.}
\label{tab:code-only-selected}
\begingroup
\small
\setlength{\tabcolsep}{2pt}
\begin{tabularx}{\columnwidth}{l*{4}{>{\raggedleft\arraybackslash}X}}
\toprule
Method & \multicolumn{2}{c}{HumanEval+} & \multicolumn{2}{c}{MBPP+} \\
\cmidrule(lr){2-3}\cmidrule(lr){4-5}
 & pass@1 & pass@64 & pass@1 & pass@64 \\
\midrule
\multicolumn{5}{l}{\textit{Initial student model}} \\
Qwen3-4B & 76.28 & 92.42 & 58.98 & 76.78 \\
\midrule
\multicolumn{5}{l}{\textit{Teachers}} \\
Code RL expert & 81.15 & 93.59 & 58.38 & 77.29 \\
Math RL expert & 78.24 & 93.08 & 60.45 & 77.23 \\
\midrule
\multicolumn{5}{l}{\textit{Without token-level routing}} \\
OPD (code teacher) & 80.90 & 93.80 & 58.39 & 77.00 \\
\specialrule{\lightrulewidth}{\aboverulesep}{0pt}
\rowcolor{routingrowcolor} \multicolumn{5}{l}{\routingtopgap\textit{With token-level routing}} \\
\rowcolor{routingrowcolor} \shortstack[l]{Entropy + TopK-KL ($\tau=0.5$)} & 80.13 & 93.81 & 59.26 & 77.92 \\
\rowcolor{routingrowcolor} \shortstack[l]{Feature + TopK-KL ($\tau=1$)}\routingbotgap & 80.03 & \textbf{94.14} & 59.21 & \textbf{78.01} \\
\specialrule{\heavyrulewidth}{0pt}{\belowbottomsep}
\end{tabularx}
\endgroup
\end{table}

For algorithmic coding, we keep the Qwen3-4B non-base variant as the initial student. For the teachers, we first train a code RL expert using more than 25000 Eurus-2-RL code prompts \citep{cui2025prime}. We also use the teacher trained via math RL from the previous synthetic task, and these two teachers form the domain experts that we use in this algorithmic coding setting, where it is reasonable to expect mathematical ability may improve algorithmic thinking and coding ability. Then, we train the student for 300 updates on code-only problems from Eurus-2-RL.

For evaluation, we use HumanEval+ and MBPP+ benchmarks \citep{chen2021code,austin2021program,liu2023evalplus}. Our results are presented in \Cref{tab:code-only-selected}, in which we compare against the baselines of the initial student model, the RL math and RL code teachers, as well as vanilla OPD. With token-level routing, the training with the feature-matching scheme shows improvements for pass@64 over all of the other baselines, despite not necessarily having higher pass@1. This pattern is consistent
with our theory of preserving coverage of successful
responses in \Cref{sec:main-coverage}. 
Finally, we also illustrate in \Cref{fig:code-rollout-colored} an example of how the feature matching scheme routes the teacher across a single student rollout.

\subsection{Interpreting Effective Routing Schemes}
We note that in the above three settings, Hellinger- or feature-based routing performed better on the code-only task and the synthetic Triton-math task, whereas the entropy-based routing performed better on the Triton-thinking task. Looking back at \Cref{eq:unified-routing-main}, we observe that when the LLM heads have similar magnitudes, the feature alignment $\rho$ is the useful router. Under the code-only and Triton-math task, we have access to teachers which are essentially fine-tuned versions of the base model. Empirically speaking, fine-tuning should not change head norms $\|\omega\|_2$ by much, and as such the feature alignment becomes the important term. Analogously, \Cref{eq:unified-routing-main} demonstrates that when the teacher features are equally (mis-)aligned with the student features, entropy becomes the useful router. Indeed, in the Triton-thinking task, neither teacher is from the same model family as the base model, so the student and the teachers naturally do not have much feature alignment. In that sense, we would expect the entropy-based routing to be more effective, which it indeed is.

\section{Conclusion}
In this work, we explore token-level teacher routing as an additional axis for OPD training. To that end, we propose a theoretical framework in the log-linear setting to unify existing classes of teacher selection methods (e.g. entropy-based) as well as a new class of representation-based routing that we propose. Together, these allow us to predict empirically when one scheme may be more beneficial than the other, in the context of the trained student's ability and output diversity. We believe there is room for much more design and improvement on this axis, as well as understanding of its downstream benefits. For instance, an interesting open question would be to understand if the diversity improvements of our teacher selection scheme can be used to create a better warm start for RL post-OPD. In addition, it would also be worthwhile to derive teacher selection metrics that are more reward-aware, beyond simply choosing the strongest teacher.

\flushbottom
\subsection*{AI use statement}

In this work, we used generative AI tools for literature search, writing code, and brainstorming ideas. We have reviewed all AI-assisted work. We take responsibility for the final content of this work,
including text, claims or artifacts produced with the aid of generative AI.

\bibliography{iclr2027_conference}
\bibliographystyle{plainnat}

\clearpage
\appendix
\crefalias{section}{appendix}
\crefalias{subsection}{subappendix}
\crefalias{subsubsection}{subsubappendix}
\thispagestyle{empty}
\onecolumn

\section{Related Work}
\label{app:related-work}

\paragraph{LLM post-training.}
Post-training LLMs with RL has generally concerned optimizing for some reward, such as a learned preference reward in RLHF \citep{christiano2023deepreinforcementlearninghuman, ouyang2022instruct} or a verifiable one in RLVR \citep{deepseek2025r1}. Beyond recent algorithmic innovations to the classic policy gradient algorithms \citep{williams1992REINFORCE, sutton1999pg} such as TRPO and PPO \citep{schulman2015trpo, schulman2017proximalpolicyoptimizationalgorithms}, a growing line of theoretical work has been aimed towards understanding what sorts of reward models yield favorable optimization guarantees \citep{razin2025rewardteacher, shang2026imperfectrewards}. In addition, there has also been an extensive line of work understanding the interplay between on- and off-policy data in training, most notably in contrasting the post-training phases of SFT and RL \citep{shenfeld2025rlsrazoronlinereinforcement, wei2026rlvrbacktrack, park2025doesrlposttraininginduce, matsutani2026rlsqueezessftexpands, huang2026emergenceimplicitcurriculumrlvr}.

\paragraph{On-policy distillation.}
On-policy distillation (OPD) replaces the scalar reward in RL with dense token-level supervision from a teacher queried on the student's own rollouts \citep{hinton2015distilling,agarwal2024gkd}. G-OPD connects vanilla OPD to KL-regularized RL and separates the reward scale from the regularization strength \citep{yang2026learningteachergeneralizedonpolicy}, and recent work by \citet{li2026sequential} finds that OPD followed by RL outperforms RL alone. A related line of work in the direction of OPD is on-policy self-distillation (OPSD), which removes the separate teacher by conditioning the same model on privileged context such as a reference solution \citep{zhao2026opsd}, with unsupervised \citep{li2026unsupervisedopsd} and shortcut-purified \citep{shen2026purifiedopsd} variants. In addition, MOPD is a way to distill from several independently trained domain experts \citep{ma2026mopdmultiteacheronpolicydistillation}. 

\paragraph{OPD dynamics and prefix drift.}
Several recent works examine why dense teacher feedback can nonetheless be hard to learn from. For instance, \citet{fu2026revisiting} separate the token-local update from the sequence-level reverse KL gradient and identify imbalanced sampled-token signals and support mismatch as failure modes; \citet{li2026rethinkingopd} find that compatible thinking patterns and genuinely new capabilities are distinct conditions for transfer; and \citet{fu2026oneshotopd} show that a single query already induces broad state visitation while the supervision itself is incorporated slowly. On the theory side, \citet{zhang2026onlineil} show that the advantage of online over offline imitation in post-training hinges on non-realizability of the expert rather than on horizon-driven error accumulation, and \citet{sriraman2026opdnoisy} show that with noisy expert feedback, OPD attains polynomial rather than exponential-in-horizon sample complexity. Additional analyses have suggested that prefixes are generally off-policy for the teacher: early-token supervision is more effective than late-token supervision \citep{xie2026positionbias}, teacher confidence decays along student rollouts \citep{liu2026fidelitydecay}, and student--teacher mismatch invites length-dependent exploitation of the objective \citep{wang2026demystifyingopd}. Our work provides an additional theoretical perspective for understanding teachers that yield good optimization guarantees under OPD.

\paragraph{Teacher selection, student compatibility, and token-level supervision.}
A more general line of work that spans the knowledge distillation literature is understanding teacher suitability for the student. For example GRACE scores teachers by the distribution of student gradients on their outputs \citep{panigrahi2025grace}, and the Rank--Surprisal Ratio proposed by \citet{yang2026informativealignment} balances alignment with the student against informativeness. 
Indeed, both of these above works choose a teacher once for the entire training duration. H-OPD \citep{yin2026hopdconfidenceawareheterogeneous} suggests the use of confidence as a way to do token-level routing, but their analysis is limited to the setting of vision tasks. Some more recent works analyze the other end of the spectrum, which is to see how little supervision one can get away with in OPD while preserving its advantages. For instance, OPSA replaces the teacher advantage on low-probability tokens \citep{ding2026opsa}, and \citet{liu2026sparsesupervision} supervise only one or two tokens per rollout. Our insight then, is that if we specialize to the on-policy training case, then we can naturally extend teacher arbitration to the token level, using different routing schemes depending on the setting.

\paragraph{Interventions for better coverage.}
The coverage principle of \citet{chen2025coverage} argues that the probability a model already assigns to high-quality responses governs what post-training and test-time selection can achieve. This framework has guided several follow-up works in terms of better interventions for improving coverage. For instance, TailSFT filters already-fit sequences out of SFT so that learning concentrates on under-modeled regions, improving pass@$k$ before RL \citep{malladi2026tailsft}. In addition, \citet{cheng2026rlreshapes} trace the gap between pass@1 gains and pass@$N$ losses after RL to a reshaping of strategy selection and connect it to coverage and transfer. The OPD-then-RL setup of \citet{li2026sequential} can also be interpreted as first broadening the set of teacher-supported solutions and then concentrating probability on the successful ones. Moreover, \citet{tuyls2026representationbasedexplorationlanguagemodels} also designs a representation-aware RL post-training scheme in order to induce better diversity of output responses. In our work, we prove that our representation-based teacher selection scheme in \Cref{sec:covariance-transfer} is coverage-aware in a similar spirit; that is, our diversity analysis in \Cref{sec:main-coverage} upper bounds the undercoverage quantity of high-quality continuations.

\section{Proofs of the Feature-Matching Teacher Selection}
\label{app:new-transfer-proofs}

We briefly describe the organization of this section below. First, in \Cref{app:reward-decomposition}, we describe how teacher competence (as measured in expected reward) and teacher-student compatibility complement each other in terms of the training outcome of OPD, via the performance difference lemma; notably, this also allows us to realize the potential of token-level teacher selection over using only a single teacher. Then, in \Cref{app:covariance-transfer-proof}, we list out the properties of the gradient and optimization landscape at initialization to work towards a proof of \Cref{thm:learning-time-abridged}. Afterwards, in \Cref{app:learning-time-proof}, we explicitly quantify the benefit of using our feature-matching selection scheme to give the full proof of \Cref{thm:learning-time-abridged}. Finally in \Cref{app:coverage-statements}, we provide the full proofs on how our selection scheme can improve the coverage of good response modes.

\subsection{Why Minimize Token-Level KL? A Performance-Difference View}
\label{app:reward-decomposition}

Recall the performance difference lemma:
\begin{align}
 J(\pi^\star)-J(\pi_\theta)
 &=\underbrace{\Delta_T}_{\text{teacher competence gap}}
   +\underbrace{J(\pi_T)-J(\pi_\theta)}_{\text{student-teacher gap}}
   \nonumber\\
 &=\Delta_T-\sum_{t=1}^H
    \mathbb E_{s_t\sim d_t^{\pi_\theta},\,v\sim\pi_\theta(\cdot\mid s_t)}
          A_t^{\pi_T}(s_t,v)\nonumber\\
 &=\Delta_T+\sum_{t=1}^H\mathbb E_{s_t\sim d_t^{\pi_\theta}}
    \sum_{v\in\mathcal V}
       \bigl(\pi_T(v\mid s_t)-\pi_\theta(v\mid s_t)\bigr)
       Q_t^{\pi_T}(s_t,v).
 \label{eq:reward-teacher-decomposition}
\end{align}

We now prove \Cref{lem:connection-to-KL} from the main text.
\begin{proof} [Proof of \Cref{lem:connection-to-KL}]
Since the terminal reward lies in $[0,1]$, its conditional expectation
$Q_t^{\pi_T}(s_t,v)$ also lies in $[0,1]$ for every prefix $s_t$
and token $v$. Moreover, the student and teacher probabilities each
sum to one. Thus, for every fixed prefix,
\begin{align*}
 &\left|\sum_{v\in\mathcal V}
       \bigl(\pi_T(v\mid s_t)-\pi_\theta(v\mid s_t)\bigr)
       Q_t^{\pi_T}(s_t,v)\right|\\
 &\quad=\left|\sum_{v\in\mathcal V}
       \bigl(\pi_T(v\mid s_t)-\pi_\theta(v\mid s_t)\bigr)
       \left(Q_t^{\pi_T}(s_t,v)-\tfrac12\right)\right|\\
 &\quad\leq \frac12\sum_{v\in\mathcal V}
       \left|\pi_T(v\mid s_t)-\pi_\theta(v\mid s_t)\right|\\
 &\quad= \operatorname{TV}\!\left(
       \pi_\theta(\cdot\mid s_t),\pi_T(\cdot\mid s_t)\right),
\end{align*}
where $\operatorname{TV}$ denotes total variation distance. By Pinsker's inequality,
\[
 \operatorname{TV}\!\left(
       \pi_\theta(\cdot\mid s_t),\pi_T(\cdot\mid s_t)\right)
 \leq \sqrt{\frac12\,
       \mathrm{KL}\!\left(\pi_\theta(\cdot\mid s_t)
                    \,\|\,\pi_T(\cdot\mid s_t)\right)}
 =\sqrt{\tfrac12 L_{T,t}(\theta;s_t)}.
\]
Finally, apply the triangle inequality to the student-teacher gap
in \Cref{eq:reward-teacher-decomposition} and use the pointwise bound:
\begin{align*}
 |J(\pi_T)-J(\pi_\theta)|
 &\leq \sum_{t=1}^H\mathbb E_{s_t\sim d_t^{\pi_\theta}}
     \left|\sum_{v\in\mathcal V}
       \bigl(\pi_T(v\mid s_t)-\pi_\theta(v\mid s_t)\bigr)
       Q_t^{\pi_T}(s_t,v)\right|\\
 &\leq \sum_{t=1}^H\mathbb E_{s_t\sim d_t^{\pi_\theta}}
       \sqrt{\tfrac12 L_{T,t}(\theta;s_t)}.
\end{align*}
This proves the claim.
\end{proof}

\subsection{Proofs for Gradients at Initialization}
\label{app:covariance-transfer-proof}

For the rest of this section, we simply fix a prefix state $s_t=(x,y_{<t})$ and prompt $x$. Recall that we are working in the setting of log-linear model families, for both the student and teachers. For the student model, we have their frozen features to be $\phi(s_t,v) \in \mathbb R^d$, and similarly for the teachers $i$, which have features $\psi_i(s_t,v) \in \mathbb R^d$ for $i=1,2$. For the student, the only trainable parameter is its linear head $\theta\in\mathbb R^d$, and let $\theta_0$ be its initial value. Similarly, $\omega_i\in\mathbb R^d$ are the
fixed heads of the teachers $i$, though those will be fixed as we do not train the teachers.

Then, we have the following policies conditioned on $s_t$:
\[
 \pi_\theta(v\mid s_t)=\frac{e^{\theta^\top\phi(s_t,v)}}{Z_S(\theta;s_t)},
 \qquad
 \pi_{T_i}(v\mid s_t)=\frac{e^{\omega_i^\top\psi_i(s_t,v)}}{Z_{T_i}(s_t)},
\]
where $Z_S$ and $Z_{T_i}$ are normalizing terms over $V$. For the teacher $i$ that we choose, the loss is the reverse KL loss:
\[
 L_i(\theta)=L_{i,t}(\theta)
 :=\sum_{v\in\mathcal V}\pi_\theta(v\mid s_t)
       \log\frac{\pi_\theta(v\mid s_t)}{\pi_{T_i}(v\mid s_t)}.
\]

When it is clear from the exposition, we will sometimes omit $s_t$ in our notations of the features or policies, or use alternative notation that will be defined when needed.

In the rest of the section, we will give a closed form expression for the gradient under our mild assumptions. Then, we will proceed to prove \Cref{thm:learning-time-abridged}.

\paragraph{Additional notation.} Recall that $\theta_0$ is the student head at initialization, $\omega_i$ is
the fixed head of teacher $i$, and $\phi(s_t,v),\psi_i(s_t,v)$
are the respective $d$-dimensional feature vectors at the same prefix $s_t$. In addition, we define the notation
\begin{align*}
    \mathbb E_t[f]:=\sum_v\pi_{\theta_0}(v\mid s_t)f(v)
\end{align*}
for the expectation of a function $f$ over the student policy $\pi_{\theta_0}$ (i.e. the student at initialization). Using this notation, we define the following feature means and covariance matrices:
\begin{align*}
 \bar\phi_t&:=\mathbb E_t[\phi(s_t,v)],
 &\bar\psi_{i,t}&:=\mathbb E_t[\psi_i(s_t,v)],\\
 F_t&:=\mathbb E_t[(\phi-\bar\phi_t)(\phi-\bar\phi_t)^\top],
 &C_i(t)&:=\mathbb E_t[(\phi-\bar\phi_t)(\psi_i-\bar\psi_{i,t})^\top].
\end{align*}

That is, $F_t$ is the student feature covariance, and $C_i(t)$ is the student-teacher cross-covariance. Furthermore, we also define notation for gradient related terms:
\[
d_i(t):=-\nabla L_i(\theta_0),\qquad
 a_i:=\|d_i(t)\|_2,
\]
as well as the mean squared feature discrepancy between the teacher and student:
\[
 D_i(t):=\mathbb E_t\|\phi(s_t,v)-\psi_i(s_t,v)\|_2^2.
\]

We make the following assumptions at initialization. Let $I_d$ denote
the $d\times d$ identity matrix, and abbreviate the features
$\phi(s_t,v),\psi_i(s_t,v)$ by $\phi,\psi_i$ inside expectations.
\begin{assumption}[Initialization whitening and teacher-student alignment]
\label{ass:initial-whitening}
At initialization, the features are centered and whitened under
$\pi_{\theta_0}(\cdot\mid s_t)$:
\[
 \mathbb E_t[\phi]=\mathbb E_t[\psi_i]=0,\qquad
 \mathbb E_t[\phi\phi^\top]=\mathbb E_t[\psi_i\psi_i^\top]=I_d.
\]
Moreover, for each teacher $i\in\{1,2\}$, we have $\mathbb E_t[\phi\psi_i^\top]=\rho_i I_d$ with
$\rho_i\in[0,1]$.
\end{assumption}

\begin{assumption}[Comparable teacher head scales]
\label{ass:aligned-teacher-heads}
 The teacher heads satisfy
$\|\omega_i\|_2=W>0$ and $\langle\omega_i,\theta_0\rangle=b$,
with the same scalars $W$ and $b$ for both teachers.
\end{assumption}

The following lemma gives the initial gradient for using either teacher in the student loss function. We then use these identities to compare the two teachers.
\begin{lemma}[Initial gradient]
\label{lem:covariance-transfer}
Under \Cref{ass:initial-whitening,ass:aligned-teacher-heads}, when the gradient is taken with respect to teacher $\pi_{T_i}$, we have:
\begin{align}
 D_i(t)&=2d(1-\rho_i),\qquad
 d_i(t)=\rho_i\omega_i-\theta_0,\nonumber\\
 a_i^2&=W^2\rho_i^2+\|\theta_0\|_2^2-2b\rho_i.
 \label{eq:normalized-transfer}
\end{align}
\end{lemma}

\begin{proof}
Recall that \[
 L_i(\theta)=L_{i,t}(\theta)
 :=\sum_{v\in\mathcal V}\pi_\theta(v\mid s_t)
       \log\frac{\pi_\theta(v\mid s_t)}{\pi_{T_i}(v\mid s_t)}.
\]

Since
$\sum_v\nabla\pi_\theta(v\mid s_t)=0$, it holds that
\begin{align*}
 \nabla L_i(\theta)
 &=\sum_v\nabla\pi_\theta(v\mid s_t)
       \left[\log\frac{\pi_\theta(v\mid s_t)}{\pi_{T_i}(v\mid s_t)}+1\right]\\
 &=\mathbb E_{\pi_\theta}\!\left[
       \nabla\log\pi_\theta(v\mid s_t)
       \log\frac{\pi_\theta(v\mid s_t)}{\pi_{T_i}(v\mid s_t)}\right].
\end{align*}
We now proceed to evaluate this gradient at initialization $\theta_0$. First, note that
\[
 \left.\nabla_\theta\log \pi_\theta(v\mid s_t)\right|_{\theta=\theta_0}
 =\phi(s_t,v)-\bar\phi_t 
\]
and
\begin{align*}
 \log\frac{\pi_{T_i}(v\mid s_t)}{\pi_{\theta_0}(v\mid s_t)}
 ={}&\omega_i^\top\psi_i(s_t,v)-\theta_0^\top\phi(s_t,v)\\
    &+\log Z_S(\theta_0;s_t)-\log Z_{T_i}(s_t).
\end{align*}

Evaluating $\nabla L_i(\theta)$ at $\theta_0$ using the above expressions, and combining with the fact that $\log Z_S(\theta_0;s_t)-\log Z_{T_i}(s_t)$ does not depend on token $v$, we have that:

\begin{equation}
 d_i(t)
 =\mathbb E_t\!\left[(\phi-\bar\phi_t)
       \log\frac{\pi_{T_i}(v\mid s_t)}{\pi_{\theta_0}(v\mid s_t)}\right]
 =C_i(t)\omega_i-F_t\theta_0.
 \label{eq:cross-feature-gradient}
\end{equation}

By assumption at initialization:
$\bar\phi_t=\bar\psi_{i,t}=0$,
$F_t=\mathbb E_t[\phi\phi^\top]=I_d$, and
$C_i(t)=\rho_i I_d$, where $\rho_i\in[0,1]$ is teacher $i$'s alignment coefficient. Then, we have that:
\[
 D_i(t)=\operatorname{tr}(I_d)+\operatorname{tr}(I_d)
           -2\operatorname{tr}(C_i(t))
       =2d(1-\rho_i),
\]
where we recall that $W=\|\omega_i\|_2$ and $b = \langle \omega_i,\theta_0 \rangle.$
Substituting $F_t=I_d$ and $C_i(t)=\rho_i I_d$ into
\Cref{eq:cross-feature-gradient}, we obtain $d_i(t)=\rho_i\omega_i-\theta_0$, from which we see that
\[
 a_i^2=\|\rho_i\omega_i-\theta_0\|_2^2
      =W^2\rho_i^2+\|\theta_0\|_2^2-2b\rho_i,
\]
as desired.
\end{proof}

Given the expressions of the gradients with respect to each teacher, we now compare them, under the following assumption which essentially says teacher 1 is more aligned to the student than teacher 2, without loss of generality.
\begin{assumption}\label{ass:wlog-closeness}
    Without loss of generality, $\rho_1 > \rho_2$. In addition, assume that $b \leq W^2 \rho_2$.
\end{assumption}

We now give the formal statement that compares the magnitudes of the norms of the gradient losses corresponding to either teacher, which exactly corresponds to the student-teacher alignment $\rho_i$ under our setup.
\begin{lemma}[Comparison of initial gradients]
\label{lem:isotropic-match-full}
Under the preceding assumptions, we have that $D_1(t)<D_2(t)$ if and only if $\rho_1>\rho_2$.
Moreover,
\begin{equation}
 a_1^2-a_2^2
 = (\rho_1-\rho_2)\bigl[W^2(\rho_1+\rho_2)-2b\bigr]>0,
 \label{eq:distance-transfer}
\end{equation} 
and hence $a_1>a_2$. 
\end{lemma}

\begin{proof}[Proof of \Cref{lem:isotropic-match-full}]

Since $D_i(t)=2d(1-\rho_i)$, smaller feature discrepancy means
larger $\rho_i$ and hence $D_1(t)<D_2(t) \iff \rho_1>\rho_2$. Furthermore,
\[
 a_1^2-a_2^2
 =W^2(\rho_1^2-\rho_2^2)-2b(\rho_1-\rho_2)
 =(\rho_1-\rho_2)[W^2(\rho_1+\rho_2)-2b].
\]
Indeed, both terms on the right hand side are positive, since $\rho_1>\rho_2$ and $W^2(\rho_1+\rho_2) > 2 W^2\rho_2 \geq 2b$.
\end{proof}

We remark that since $\langle d_i(t),\omega_i\rangle = W^2\rho_i-b$, we are in the regime where this inner product is nonnegative, in which case increasing teacher-student alignment corresponds with this increasing gradient norm. Although this is a statement about initialization, we will show in the next section that under certain suitable assumptions, this yields a separation in the rate of optimization.

\subsection{Proofs for Learning-Time Separation}
\label{app:learning-time-proof}

Fix a prefix $s_t=(x,y_{<t})$. Recall the log-linear student and
teachers from \Cref{sec:unified-routing}:
\[
 \pi_\theta(v\mid s_t)
 =\frac{\exp(\theta^\top\phi(s_t,v))}{Z_S(\theta;s_t)},
 \qquad
 \pi_{T_i}(v\mid s_t)
 =\frac{\exp(\omega_i^\top\psi_i(s_t,v))}{Z_{T_i}(s_t)},
 \qquad i\in\{1,2\},
\]
as well as our loss functions:
\[
 L_i(\theta):=\mathrm{KL}\!\left(
       \pi_\theta(\cdot\mid s_t)\,\|\,\pi_{T_i}(\cdot\mid s_t)\right).
\]
Starting from the same student head $\theta_0$, gradient descent
with a common learning rate $\eta>0$ gives
\[
 \theta_0^{(i)}=\theta_0,\qquad
 \theta_{n+1}^{(i)}
 =\theta_n^{(i)}-\eta\nabla L_i(\theta_n^{(i)}).
\]
Recall our notation $a_i:=\|\nabla L_i(\theta_0)\|_2$ for the initial gradient norm.
For a target loss level $0<\varepsilon<\min_i L_i(\theta_0)$, define
\[
 \delta_i:=L_i(\theta_0)-\varepsilon,\qquad
 N_i(\varepsilon):=
 \inf\{n\in\mathbb Z_{\geq0}:L_i(\theta_n^{(i)})\leq\varepsilon\}.
\]

We begin with the following bounded features assumption.
\begin{assumption}[Bounded student and teacher features]
\label{ass:bounded-student-features}
For every candidate token $v$ at the fixed prefix and each teacher
$i\in\{1,2\}$,
\[
 \|\phi(s_t,v)\|_2\leq B,\qquad \|\psi_i(s_t,v)\|_2\leq B
\]
for some $B>0$.
\end{assumption}

In the rest of this section, we will work with learning rates $\eta$ that satisfy $0<\eta\Lambda<1$, and we denote $r:=\eta \Lambda$. We first introduce the following additional notation:
for teacher $i$, let
$h_{i,n}:=\|\nabla L_i(\theta_n^{(i)})\|_2$ be the gradient norm
after $n$ updates, and let
$P_i(n):=L_i(\theta_0)-L_i(\theta_n^{(i)})$ be the total loss
reduction after those updates.

\begin{lemma}[Two-sided loss-reduction bounds]
\label{lem:initial-gradient-loss-progress}
If $\nabla L_i$ is $\Lambda$-Lipschitz and $0<r=\eta\Lambda<1$,
then, for every integer $n\geq0$,
\[
 (1-r)^n a_i\leq h_{i,n}\leq(1+r)^n a_i
\]
and
\begin{equation}
 \frac{a_i^2}{2\Lambda}\bigl[1-(1-r)^{2n}\bigr]
 \leq P_i(n)\leq
 \frac{a_i^2}{2\Lambda}\bigl[(1+r)^{2n}-1\bigr].
 \label{eq:two-sided-loss-progress}
\end{equation}
\end{lemma}
\begin{proof}
First, we have by $\Lambda$-smoothness that:
\[
 |h_{i,n+1} - h_{i,n}| \leq \|\nabla L_i(\theta_{n+1}^{(i)})
       -\nabla L_i(\theta_n^{(i)})\|_2
 \leq\Lambda\|\theta_{n+1}^{(i)}-\theta_n^{(i)}\|_2
 =r h_{i,n}.
\]
By triangle inequality, we have
$(1-r)h_{i,n}\leq h_{i,n+1}\leq(1+r)h_{i,n}$, from which induction yields $ (1-r)^n a_i\leq h_{i,n}\leq(1+r)^n a_i$.

Furthermore, for a $\Lambda$-smooth loss, the error in its first-order Taylor
expansion along an update is at most
$\Lambda\eta^2h_{i,n}^2/2$ in absolute value. Consequently,
\[
 \eta(1-r/2)h_{i,n}^2
 \leq L_i(\theta_n^{(i)})-L_i(\theta_{n+1}^{(i)})
 \leq\eta(1+r/2)h_{i,n}^2.
\]
Using the above gradient bound, we sum and telescope these inequalities over $j=0,\ldots,n-1$ to obtain
\begin{align*}
 P_i(n)
 &\geq\eta(1-r/2)a_i^2\sum_{j=0}^{n-1}(1-r)^{2j}
   =\frac{a_i^2}{2\Lambda}[1-(1-r)^{2n}],\\
 P_i(n)
 &\leq\eta(1+r/2)a_i^2\sum_{j=0}^{n-1}(1+r)^{2j}
   =\frac{a_i^2}{2\Lambda}[(1+r)^{2n}-1],
\end{align*}
as desired.
\end{proof}

\begin{proposition}[Learning time bounds from the initial gradient]
\label{prop:initial-gradient-learning-time}
For $a_i>0$, $0<r=\eta\Lambda<1$, and
$\delta_i=L_i(\theta_0)-\varepsilon>0$,
\begin{equation}
 N_i(\varepsilon)\geq
 \left\lceil
 \frac{\log(1+2\Lambda\delta_i/a_i^2)}
      {2\log(1+r)}
 \right\rceil.
 \label{eq:loss-progress-time-lower}
\end{equation}
If additionally $0<2\Lambda\delta_i/a_i^2<1$, then
\begin{equation}
 N_i(\varepsilon)\leq
 \left\lceil
 \frac{\log(1-2\Lambda\delta_i/a_i^2)}
      {2\log(1-r)}
 \right\rceil.
 \label{eq:loss-progress-time-upper}
\end{equation}
\end{proposition}
\begin{proof}
We first analyze the lower bound. Note that reaching the target requires $P_i(n)\geq\delta_i$. From the upper bound in
\Cref{eq:two-sided-loss-progress}, we must have
$(1+r)^{2n}\geq1+2\Lambda\delta_i/a_i^2$ to achieve this.
Taking logarithms and rounding up yields the lower bound on $n$ to be the right hand side of \Cref{eq:loss-progress-time-lower}.

For the upper bound, it suffices that the lower bound on $P_i(n)$
is at least $\delta_i$. When $2\Lambda\delta_i/a_i^2<1$, this holds
as soon as
\[
 (1-r)^{2n}\leq1-\frac{2\Lambda\delta_i}{a_i^2}.
\]
By a similar manipulation, we obtain that $n$ must be at least the bound in
\Cref{eq:loss-progress-time-upper}.
\end{proof}

We now proceed to the proof of the main theorem.

\begin{theorem}[Full version of \Cref{thm:learning-time-abridged}]
\label{thm:learning-time-full}

Let $x_i:=2\Lambda\delta_i/a_i^2$, and suppose that for a fixed constant $\kappa\in(0,1)$, we have that $\eta,\epsilon$ satisfy:
\[
 0<r=\eta\Lambda\leq\kappa,\qquad
 0<x_i=\frac{2\Lambda[L_i(\theta_0)-\varepsilon]}{a_i^2}
       \leq\kappa,\quad i\in\{1,2\}.
\]
Then, for each teacher $i$,
\begin{equation}
 \left\lceil\frac{\delta_i}{(1+\kappa)\eta a_i^2}\right\rceil
 \leq N_i(\varepsilon)\leq
 \left\lceil\frac{\delta_i}{(1-\kappa)\eta a_i^2}\right\rceil.
 \label{eq:inverse-square-learning-time}
\end{equation}
In particular,
\[
 N_1(\varepsilon)
 =O\!\left(1+\frac{\delta_1}{\eta a_1^2}\right),
 \qquad
 N_2(\varepsilon)
 =\Omega\!\left(\frac{\delta_2}{\eta a_2^2}\right),
\]
with constants depending only on $\kappa$, within the stated regime.
\end{theorem}
\begin{proof}
For $0\leq z\leq\kappa<1$, integration of $1/(1-z)$ and
$1/(1+z)$ gives
\[
 z\leq-\log(1-z)\leq\frac{z}{1-\kappa},
 \qquad
 \frac{z}{1+\kappa}\leq\log(1+z)\leq z.
\]
Apply the first pair to $x_i$ and $r$ in the upper bound:
\[
 \frac{-\log(1-x_i)}{2[-\log(1-r)]}
 \leq\frac{x_i}{2(1-\kappa)r}
 =\frac{\delta_i}{(1-\kappa)\eta a_i^2}.
\]
Apply the second pair to the lower bound:
\[
 \frac{\log(1+x_i)}{2\log(1+r)}
 \geq\frac{x_i}{2(1+\kappa)r}
 =\frac{\delta_i}{(1+\kappa)\eta a_i^2}.
\]
Taking ceilings and using
\Cref{prop:initial-gradient-learning-time} proves the result.
\end{proof}

\subsection{Proofs for Diversity and Undercoverage}
\label{app:coverage-statements}

We first compute the part of the teacher's logit signal that the student features cannot represent, in the following sense. Define the teacher--student logit gap and its variance under the initial student below:
\[
 r_{i,\theta}(v):=\omega_i^\top\psi_i(v)-\theta^\top\phi(v),
 \qquad
 R_i(\theta):=\operatorname{Var}_{\pi_{\theta_0}}[r_{i,\theta}(v)],
\]

Then, the following lemma quantifies this representational gap through the structure at initialization.

\begin{lemma}[Residual logit variance from feature alignment]
\label{lem:alignment-residual-variance}
Let $\bar \theta_i := \rho_i\omega_i$ and $E_i:=R_i(\bar\theta_i)$. Under the preceding assumptions, it holds that
\begin{equation}
 R_i(\theta)=\|\theta-\bar\theta_i\|_2^2+E_i,
 \qquad
 E_i=\min_{\theta\in\mathbb R^d}R_i(\theta)
     =W^2(1-\rho_i^2).
 \label{eq:feature-logit-projection}
\end{equation}
In particular,
\[
 \operatorname{Var}_{\pi_{\theta_0}}[\omega_i^\top\psi_i]=W^2,
 \qquad
 \operatorname{Var}_{\pi_{\theta_0}}[\bar\theta_i^\top\phi]
       =\rho_i^2W^2,
 \qquad
 \mathbb E_{\pi_{\theta_0}}[\phi\,r_{i,\bar\theta_i}]=0.
\]
Thus the candidate captures a fraction $\rho_i^2$ of the teacher's
logit variance, leaving a fraction $1-\rho_i^2$ outside the student
feature span.
\end{lemma}

\begin{proof}
Unless noted otherwise, all expectations in this section are taken over the initial student policy $\pi_{\theta_0}$. We calculate as follows:
\begin{align*}
 R_i(\theta)
 &=\mathbb E[(\omega_i^\top\psi_i-\theta^\top\phi)^2]\\
 &=\omega_i^\top\mathbb E[\psi_i\psi_i^\top]\omega_i
       -2\theta^\top\mathbb E[\phi\psi_i^\top]\omega_i
       +\theta^\top\mathbb E[\phi\phi^\top]\theta\\
 &=W^2-2\rho_i\theta^\top\omega_i+\|\theta\|_2^2\\
 &=\|\theta-\rho_i\omega_i\|_2^2+W^2(1-\rho_i^2).
\end{align*}
In particular, the first term is minimized at $\bar\theta_i=\rho_i\omega_i$,
which proves \Cref{eq:feature-logit-projection}.
The teacher-logit and candidate-logit variances are respectively
$\|\omega_i\|_2^2=W^2$ and
$\|\bar\theta_i\|_2^2=\rho_i^2W^2$.
Finally,
\[
 \mathbb E[\phi\,r_{i,\bar\theta_i}]
 =\mathbb E[\phi\psi_i^\top]\omega_i
       -\mathbb E[\phi\phi^\top]\bar\theta_i
 =\rho_i\omega_i-\bar\theta_i=0.
\]
\end{proof}

We next convert the residual variance into a bound on reverse KL.
For a student head $\theta$ and an interpolation parameter
$\lambda\in[0,1]$, define
\[
 \pi_{i,\theta,\lambda}(v):=
 \frac{\exp\{(1-\lambda)\theta^\top\phi(v)
                       +\lambda\omega_i^\top\psi_i(v)\}}
      {\sum_{u\in\mathcal V}\exp\{(1-\lambda)\theta^\top\phi(u)
                       +\lambda\omega_i^\top\psi_i(u)\}}.
\]
Here, $\lambda=0$ and $\lambda=1$ correspond to policies $\pi_\theta$ and $\pi_{T_i}$, respectively. The following lemma allows us to use properties of this interpolation policy to bound the reverse KL.

\begin{lemma}[Reverse KL is controlled by residual logit variance]
\label{lem:residual-variance-kl}
For a student head $\theta$, suppose $\Gamma\geq1$ satisfies the following probability control condition:
\begin{equation}
 \pi_{i,\theta,\lambda}(v)\leq\Gamma\pi_{\theta_0}(v)
 \quad\text{for every }v\in\mathcal V,\ \lambda\in[0,1].
 \label{eq:kl-overlap-upper-density}
\end{equation}
Then
\begin{equation}
 L_i(\theta)
 \leq\frac{\Gamma}{2}
       \operatorname{Var}_{\pi_{\theta_0}}
         [\omega_i^\top\psi_i-\theta^\top\phi]
 =\frac{\Gamma}{2}R_i(\theta).
 \label{eq:variance-to-kl-upper}
\end{equation}
In particular, if the condition holds at
$\bar\theta_i=\rho_i\omega_i$, then
\begin{equation}
 L_i^\star := \inf_\theta L_i(\theta) \leq L_i(\bar\theta_i)
 \leq\frac{\Gamma}{2}E_i
 =\frac{\Gamma W^2(1-\rho_i^2)}{2}.
 \label{eq:candidate-variance-kl}
\end{equation}
\end{lemma}

\begin{proof}
Recall that $r_{i,\theta}(v)=\omega_i^\top\psi_i(v)
-\theta^\top\phi(v)$ is the logit error. For the fixed head $\theta$,
define the scalar log-partition function
\[
 A_{i,\theta}(\lambda)
 :=\log\sum_v\exp\{\theta^\top\phi(v)+\lambda r_{i,\theta}(v)\}.
\]
Differentiating yields
\[
 A_{i,\theta}'(\lambda)
 =\mathbb E_{\pi_{i,\theta,\lambda}}[r_{i,\theta}],
 \qquad
 A_{i,\theta}''(\lambda)
 =\operatorname{Var}_{\pi_{i,\theta,\lambda}}[r_{i,\theta}].
\]
Substituting these quantities into the reverse KL loss and using Taylor expansion, we have that
\begin{align}
 L_i(\theta)
 &=A_{i,\theta}(1)-A_{i,\theta}(0)-A_{i,\theta}'(0)\nonumber\\
 &=\int_0^1(1-\lambda)
       \operatorname{Var}_{\pi_{i,\theta,\lambda}}[r_{i,\theta}]
       \,d\lambda.
 \label{eq:kl-overlap-integral}
\end{align}

To bound the path variance, note that:
\begin{align*}
 \operatorname{Var}_{\pi_{i,\theta,\lambda}}[r_{i,\theta}]
 &\leq\mathbb E_{\pi_{i,\theta,\lambda}}\!\left[
       \bigl(r_{i,\theta}
          -\mathbb E_{\pi_{\theta_0}}[r_{i,\theta}]\bigr)^2\right]\\
 &\leq\Gamma\mathbb E_{\pi_{\theta_0}}\!\left[
       \bigl(r_{i,\theta}
          -\mathbb E_{\pi_{\theta_0}}[r_{i,\theta}]\bigr)^2\right]\\
 &=\Gamma R_i(\theta),
\end{align*}
where the second inequality uses the pointwise probability bound from earlier. Substituting this back into \Cref{eq:kl-overlap-integral} yields
\[
 L_i(\theta)\leq
 \Gamma R_i(\theta)\int_0^1(1-\lambda)\,d\lambda
 =\frac{\Gamma}{2}R_i(\theta).
\]

In particular, at $\bar\theta_i$,
\Cref{lem:alignment-residual-variance} gives
$R_i(\bar\theta_i)=E_i=W^2(1-\rho_i^2)$.
Since $L_i^\star=\inf_\theta L_i(\theta)\leq L_i(\bar\theta_i)$,
\Cref{eq:candidate-variance-kl} follows. 
\end{proof}

In the next part, we will translate the above bound on the KL minimizer into a lower bound on the diversity of the high-quality modes that the student will learn. Recall that $G_1,\ldots,G_M\subseteq\mathcal V$ are disjoint,
nonempty next-token sets representing $M$ successful continuation
modes. We now make the following assumption about the distribution of the modes from the teacher policy.

\begin{assumption}[Uniformity of high-quality modes]\label{ass:uniform-modes}
Assume that selecting a token in
$G_j$ leads to success in mode $j$, and selecting outside their union leads
to failure. Each teacher has success probability $c\in(0,1)$ (i.e. the probability of sampling a good mode) and
\begin{equation}
 \pi_{T_i}(G_j)=c/M,\qquad j=1,\ldots,M.
 \label{eq:uniform-good-modes}
\end{equation}
\end{assumption}

In addition, we make the following technical assumption that allows us to obtain non-vacuous guarantees in any of the diversity quantities in the main theorem.
\begin{assumption}[Probability control and small teacher heads]
\label{ass:coverage-probability-small-head}
Let $\Gamma$ be defined as in \Cref{lem:residual-variance-kl}. Then, the norm of the teacher head satisfies:
\begin{equation}
 0<W<\frac{2c}{M\sqrt{\Gamma}}.
 \label{eq:coverage-small-head}
\end{equation}
\end{assumption}

For a policy $\pi$ and an integer $k\geq1$ counting independent samples,
the expected number of distinct successful modes is
\[
 \operatorname{Distinct}_k(\pi)
 :=\sum_{j=1}^M[1-(1-\pi(G_j))^k].
\]
In the following theorem, we will analyze this quantity.

\begin{theorem}[Diversity guarantee at a KL minimum, full statement]
\label{thm:coverage-full}
Under the preceding assumptions, let $\theta_i^\star$ be a finite
global minimizer of $L_i(\theta)=\mathrm{KL}(\pi_\theta\|\pi_{T_i})$ for each teacher $i\in\{1,2\}$. Define the teacher-dependent error bound
\begin{equation}
 e_i:=\frac W2\sqrt{\Gamma(1-\rho_i^2)}.
 \label{eq:coverage-error-definition}
\end{equation}
Then $0\leq e_i<c/M$, and for every $j\in\{1,\ldots,M\}$ and
integer $k\geq1$,
\begin{align}
 \operatorname{TV}(\pi_{\theta_i^\star},\pi_{T_i})&\leq e_i,
 \nonumber\\
 \pi_{\theta_i^\star}(G_j)&\geq c/M-e_i>0,\nonumber\\
 \operatorname{Distinct}_k(\pi_{\theta_i^\star})
 &\geq M\bigl[1-(1-c/M+e_i)^k\bigr].
 \label{eq:coverage-full}
\end{align}
\end{theorem}

\begin{proof}[Proof of \Cref{thm:coverage-full}]
Fix a teacher $i$ and prefix state $s_t$. Then, we have from \Cref{lem:alignment-residual-variance,lem:residual-variance-kl} that
\begin{equation}
 \mathrm{KL}(\pi_{\theta_i^\star}\|\pi_{T_i})
 =L_i^\star
 \leq L_i(\bar\theta_i)
 \leq\frac{\Gamma E_i}{2}
 =2e_i^2.
 \label{eq:coverage-proof-kl-chain}
\end{equation}

Pinsker's inequality therefore implies
\[
 \operatorname{TV}(\pi_{\theta_i^\star},\pi_{T_i})
 \leq\sqrt{\frac{L_i^\star}{2}}
 \leq e_i.
\]
Since total variation bounds the probability difference of any
event, each good mode satisfies
\[
 \pi_{\theta_i^\star}(G_j)
 \geq\pi_{T_i}(G_j)
       -\operatorname{TV}(\pi_{\theta_i^\star},\pi_{T_i})
 \geq c/M-e_i>0.
\]

In $k$ independent samples, mode $j$ is observed at least once with
probability $1-(1-\pi_{\theta_i^\star}(G_j))^k$.
Summing these probabilities gives the expected number of distinct
good modes, so
\begin{align}
 \operatorname{Distinct}_k(\pi_{\theta_i^\star})
 &=\sum_{j=1}^M
       \bigl[1-(1-\pi_{\theta_i^\star}(G_j))^k\bigr]\nonumber\\
 &\geq M\bigl[1-(1-c/M+e_i)^k\bigr].
 \label{eq:feature-coverage-guarantees}
\end{align}
\end{proof}

We remark that the error $e_i$ decreases monotonically with the alignment quantity $\rho_i$. Hence, $D_1(t) < D_2(t)$ gives teacher 1 a stronger lower bound in the diversity. As a corollary to this, we consider the following quantity, the \textit{undercoverage}, as adapted from \citet{chen2025coverage}. 

Recall the target distribution $\nu(j)=1/M$, uniform over the
$M$ good modes, and the threshold $N>1/c$.
The quantity $C_N^{\rm mode}$ in \Cref{eq:mode-undercoverage-definition}
is the fraction of these modes with student probability at most
$\nu(j)/N=1/(MN)$. 

\begin{proof}[Proof of \Cref{cor:feature-undercoverage}]
Fix $N>1/c$ and define
\[
 \mathcal A_N:=\{j:\pi_{\theta_i^\star}(G_j)\leq1/(MN)\},
 \qquad B_N:=\bigcup_{j\in\mathcal A_N}G_j.
\]
Let the golden reference distribution on the $M$ good modes be uniform; that is, $\nu(j)=1/M$. Hence the undercoverage probability is
$C_N^{\rm mode}(\nu\|\pi_{\theta_i^\star})
 =|\mathcal A_N|/M$, the fraction of modes in $\mathcal A_N$.
The sets $G_j$ are disjoint, so
\[
 \pi_{T_i}(B_N)=\frac{c|\mathcal A_N|}{M},
 \qquad
 \pi_{\theta_i^\star}(B_N)\leq\frac{|\mathcal A_N|}{MN}.
\]
From the proof of \Cref{thm:coverage-full}, we established
$\operatorname{TV}(\pi_{\theta_i^\star},\pi_{T_i})\leq e_i$.
Applying this bound to the single event $B_N$ gives
\[
 \left(c-\frac1N\right)\frac{|\mathcal A_N|}{M}
 \leq \pi_{T_i}(B_N)-\pi_{\theta_i^\star}(B_N)
 \leq e_i,
\]
and the result follows via straightforward manipulation.
\end{proof}

\subsection{Helper Lemmas}

In this section, we list out explicit instantiations of the problem-relevant parameters in our setting, such as $\Gamma$ and $\Lambda$. As these are all proven to be finite values, the exact quantity of them is irrelevant for our main proofs.

\begin{lemma}[Smoothness of the reverse KL loss]
\label{lem:actual-teacher-smoothness}
$\nabla L_i$ is $\Lambda_i$-Lipschitz, where
\begin{equation}
 \Lambda_i=\Lambda:=B^2\bigl[1+5(\log|\mathcal V|+WB)\bigr].
 \label{eq:actual-teacher-smoothness}
\end{equation}
\end{lemma}

\begin{proof}
We will bound the operator norm $\|\nabla^2 L_i(\theta)\|_{\mathrm{op}}$.
First, note that $|\omega_i^\top\psi_i(s_t,v)|\leq WB$. For a student head $\theta,$ define
\[
 z_\theta(v):=\phi(s_t,v)-\mathbb E_{\pi_\theta}[\phi(s_t,v)],
 \qquad
 \ell_{i,\theta}(v):=
       \log\frac{\pi_\theta(v\mid s_t)}{\pi_{T_i}(v\mid s_t)}.
\]
Denote the covariance matrix $F_t(\theta)=\mathbb E_{\pi_\theta}[z_\theta z_\theta^\top]$. Recall that $\nabla L_i=\mathbb E_{\pi_\theta}[z_\theta\ell_{i,\theta}]$. Then, we differentiate to obtain:
\begin{equation}
 \nabla^2 L_i(\theta)
 =F_t(\theta)
   +\mathbb E_{\pi_\theta}\!\left[
       \bigl(z_\theta z_\theta^\top-F_t(\theta)\bigr)
          \ell_{i,\theta}\right],
 \label{eq:actual-teacher-hessian}
\end{equation}

Let $\|\cdot\|_{\rm op}$ denote the Euclidean operator norm. Since $\|\phi\|_2 \leq B$, it holds that $\|F_t(\theta)\|_{\rm op}\leq B^2$ and $\|z_\theta(v)\|_2\leq2B$. To bound the operator norm of $\mathbb E_{\pi_\theta}\!\left[
       \bigl(z_\theta z_\theta^\top-F_t(\theta)\bigr)
          \ell_{i,\theta}\right]$, we calculate:
\begin{align*}
    &\left\| \mathbb E_{\pi_\theta}\!\left[
       \bigl(z_\theta z_\theta^\top-F_t(\theta)\bigr)
          \ell_{i,\theta}\right] \right\|_{\rm op} \\
          &= \left\|\mathbb E_{\pi_\theta}\!\left[
       \bigl(z_\theta(v)z_\theta(v)^\top-F_t(\theta)\bigr)
       \bigl(\log\pi_\theta(v\mid s_t)-\omega_i^\top\psi_i(s_t,v)\bigr)
       \right]\right\|_{\rm op} \\
       &\leq \mathbb E_{\pi_\theta}\!\left[
       \left\|
       \bigl(z_\theta(v)z_\theta(v)^\top-F_t(\theta)\bigr)
       \bigl(\log\pi_\theta(v\mid s_t)-\omega_i^\top\psi_i(s_t,v)\bigr)
       \right\|_{\rm op}\right]\\
       &=\mathbb E_{\pi_\theta}\!\left[
       |\log\pi_\theta(v\mid s_t)-\omega_i^\top\psi_i(s_t,v)|
       \,\|z_\theta(v)z_\theta(v)^\top-F_t(\theta)\|_{\rm op}
       \right]\\
 &\leq5B^2\mathbb E_{\pi_\theta}
       |\log\pi_\theta(v\mid s_t)-\omega_i^\top\psi_i(s_t,v)|\\
 &\leq5B^2(\log|\mathcal V|+WB),
\end{align*}
where the second to last inequality follows from $\|z_\theta(v)z_\theta(v)^\top-F_t(\theta)\|_{\rm op} \leq 5B^2$, and the last inequality follows from the following bound on the scale factor:
\begin{align*}
 \mathbb E_{\pi_\theta}
       |\log\pi_\theta(v\mid s_t)-\omega_i^\top\psi_i(s_t,v)|
 &=\sum_{v\in\mathcal V}\pi_\theta(v\mid s_t)
       |\log\pi_\theta(v\mid s_t)-\omega_i^\top\psi_i(s_t,v)|\\
 &\leq\sum_{v\in\mathcal V}\pi_\theta(v\mid s_t)
       \bigl(|\log\pi_\theta(v\mid s_t)|
                       +|\omega_i^\top\psi_i(s_t,v)|\bigr)\\
 &\leq\sum_{v\in\mathcal V}\pi_\theta(v\mid s_t)
       \bigl(-\log\pi_\theta(v\mid s_t)+WB\bigr)\\
 &=H\bigl(\pi_\theta(\cdot\mid s_t)\bigr)+WB\\
 &\leq\log|\mathcal V|+WB.
\end{align*}

Finally, by triangle inequality, we have
\begin{align*}
    \|\nabla^2 L_i(\theta)\|_{\rm op} \leq \| F_t(\theta)\|_{\rm op} + \| \E\qty[ \bigl(z_\theta z_\theta^\top-F_t(\theta)\bigr)
          \ell_{i,\theta} ] \|_{\rm op} \leq B^2(1 + 5(\log|\mathcal V| + WB)),
\end{align*}
as desired.
\end{proof}

\begin{lemma}[Explicit probability control constants]
\label{lem:explicit-probability-control}
For a real function $f$ on the vocabulary $\mathcal V$, define
$\operatorname{osc}(f):=\max_v f(v)-\min_v f(v)$.
Abbreviating the student and teacher features by $\phi(v),\psi_i(v)$ respectively, let
\[
 g_{i,\theta}^{S}(v):=(\theta-\theta_0)^\top\phi(v),
 \qquad
 g_i^{T}(v):=\omega_i^\top\psi_i(v)-\theta_0^\top\phi(v),
\]
and
\[
 \delta_i(\theta):=
 \max\!\left\{\operatorname{osc}(g_{i,\theta}^{S}),
               \operatorname{osc}(g_i^{T})\right\}.
\]
Then the interpolation $\pi_{i,\theta,\lambda}$ satisfies
\begin{equation}
 e^{-\delta_i(\theta)}
 \leq\frac{\pi_{i,\theta,\lambda}(v)}{\pi_{\theta_0}(v)}
 \leq e^{\delta_i(\theta)}
 \quad\text{for every }v\in\mathcal V,\ \lambda\in[0,1],
 \label{eq:overlap-oscillation-control}
\end{equation}
where the oscillations are bounded by
\begin{equation}
 0\leq\delta_i(\theta)
 \leq2B\max\!\left\{\|\theta-\theta_0\|_2,\,
                         W+\|\theta_0\|_2\right\}.
 \label{eq:overlap-finite-oscillation}
\end{equation}
In particular, for the candidate heads
$\bar\theta_i=\rho_i\omega_i$ with $\rho_i\in[0,1]$, we can choose
\begin{equation}
 \Gamma_{\rm cand}:=
 \exp\!\left(\max_{i\in\{1,2\}}\delta_i(\bar\theta_i)\right)
 \leq\exp\!\left\{2B(W+\|\theta_0\|_2)\right\}<\infty.
 \label{eq:overlap-candidate-gamma}
\end{equation}
to satisfy $\Gamma_{\rm cand}^{-1}\pi_{\theta_0}(v)
\leq\pi_{i,\bar\theta_i,\lambda}(v)
\leq\Gamma_{\rm cand}\pi_{\theta_0}(v)$ for both teachers and all $v,\lambda.$
\end{lemma}

\begin{proof}
Let $ g_{i,\theta,\lambda}(v)
 :=(1-\lambda)g_{i,\theta}^{S}(v)+\lambda g_i^{T}(v).$ Then, by normalization, we obtain
\begin{equation}
 \frac{\pi_{i,\theta,\lambda}(v)}{\pi_{\theta_0}(v)}
 =\frac{\exp g_{i,\theta,\lambda}(v)}
        {\mathbb E_{\pi_{\theta_0}}\exp g_{i,\theta,\lambda}}.
 \label{eq:overlap-exact-density-ratio}
\end{equation}
For $\lambda\in[0,1]$, the oscillation of this convex combination
satisfies
\begin{align*}
 \operatorname{osc}(g_{i,\theta,\lambda})
 &\leq(1-\lambda)\operatorname{osc}(g_{i,\theta}^{S})
        +\lambda\operatorname{osc}(g_i^{T}) \leq\delta_i(\theta).
\end{align*}
Also, because the denominator is an expectation,
\[
 e^{\min_u g_{i,\theta,\lambda}(u)}
 \leq\mathbb E_{\pi_{\theta_0}}e^{g_{i,\theta,\lambda}}
 \leq e^{\max_u g_{i,\theta,\lambda}(u)}.
\]
Combining these inequalities with
\Cref{eq:overlap-exact-density-ratio} bounds the ratio between
$e^{-\operatorname{osc}(g_{i,\theta,\lambda})}$ and
$e^{\operatorname{osc}(g_{i,\theta,\lambda})}$, proving
\Cref{eq:overlap-oscillation-control}. Finally, to show that $\Gamma_\mathrm{cand}$ satisfies the condition, note that Cauchy--Schwarz and the
feature bounds give, for every token,
\[
 |g_{i,\theta}^{S}(v)|
 \leq B\|\theta-\theta_0\|_2,
 \qquad
 |g_i^{T}(v)|
 \leq B(W+\|\theta_0\|_2),
\]
and by triangle inequality we have
\[
 \operatorname{osc}(g_{i,\theta}^{S})
 \leq2B\|\theta-\theta_0\|_2,
 \qquad
 \operatorname{osc}(g_i^{T})
 \leq2B(W+\|\theta_0\|_2),
\]
which proves \Cref{eq:overlap-finite-oscillation}.
At the candidate heads,
\[
 \|\bar\theta_i-\theta_0\|_2
 \leq\rho_i\|\omega_i\|_2+\|\theta_0\|_2
 \leq W+\|\theta_0\|_2.
\]
Substituting into the oscillation bound and taking the maximum over
the two teachers proves \Cref{eq:overlap-candidate-gamma}.
\end{proof}

\section{Proofs of Confidence-Based Routing Selection}
\label{app:confidence-routing}

In this section, we give results for understanding when entropy-based teacher routing can lead to good learnability. In \Cref{app:alignment-confidence}, we will see how this selection method is actually complementary to the representation-based selection method from before. As before, we are not concerned with the reward, as the contribution of that is deferred to the corresponding term of the performance difference lemma in \Cref{app:reward-decomposition}.

We continue with similar notation from previous sections, fixing a prefix $s_t=(x,y_{<t})$. One difference, however, is that for technical reasons, we will assume the two teachers share the same feature map. Let $\phi(s_t,v)\in\mathbb R^d$ be the frozen student feature map, and let $\psi(s_t,v)\in\mathbb R^d$ be a frozen feature map shared by the two teachers. We abbreviate these maps by $\phi(v)$ and $\psi(v)$ at
the fixed prefix. Then, the policies are \[
 \pi_\theta(v)
 =\frac{\exp(\theta^\top\phi(v))}
        {\sum_{z\in\mathcal V}\exp(\theta^\top\phi(z))},
 \qquad
 \pi_{T_i}(v)
 =\frac{\exp(\omega_i^\top\psi(v))}
        {\sum_{z\in\mathcal V}\exp(\omega_i^\top\psi(z))}.
\]

As before, let $\theta_0$ be the student head at initialization, and let $\omega_i$ be the fixed head of teacher $i$. Define once again the following:
\[
 L_i(\theta):=\mathrm{KL}(\pi_\theta\|\pi_{T_i}),\qquad
 d_i:=-\nabla L_i(\theta_0),\qquad
 a_i:=\|d_i\|_2.
\]

In addition, we will write
$H(\pi):=-\sum_v\pi(v)\log\pi(v)$ for entropy of $\pi$ at this fixed prefix and
$H_i:=H(\pi_{T_i})$ for teacher $i$'s entropy.

\begin{assumption}[Equal feature alignment and teacher head directions]
\label{ass:confidence-common-family}
In addition to \Cref{ass:initial-whitening} from before, we will also assume that the teacher heads have a common unit direction $u\in\mathbb R^d$:
\[
 \omega_i=W_i u,\qquad \|u\|_2=1,\qquad W_i>0,\quad i\in\{1,2\}.
\]
We also denote the cross-covariance scalar multiple as $\rho\in(0,1]$, where
$\mathbb E_{\pi_{\theta_0}}[\phi\psi^\top]=\rho I_d$, where we define $\rho := \rho_1=\rho_2$ and $\psi:=\psi_1=\psi_2$ from the setting of \Cref{ass:initial-whitening}.
\end{assumption}

As a technical assumption similar to the previous analysis, we will assume the following about the inner product between the teacher and student heads at initialization.

\begin{assumption}\label{ass:confidence-beyond-student}
Let $\beta:=\langle u,\theta_0\rangle.$ Then, $\rho\min\{W_1,W_2\}\geq\beta$.
\end{assumption}

We now prove the following theorem.

\begin{theorem}[Comparison of initial gradients]
\label{thm:confidence-gradient}
Under the preceding assumptions, we have that $H_1 < H_2$ if and only if $W_1 > W_2$. Moreover, 
\begin{equation}
 a_1^2-a_2^2
 =\rho(W_1-W_2)\bigl[\rho(W_1+W_2)-2\beta\bigr],
 \label{eq:confidence-gradient-difference}
\end{equation}
which is positive if and only if $a_1 > a_2$.
\end{theorem}

\begin{proof}
Let the common scalar logit profile be $f(v):=u^\top\psi(v)$.
For a scalar head magnitude $w>0$, write
\[
 A(w):=\log\sum_v e^{wf(v)},\qquad
 \pi_w^T(v):=e^{wf(v)-A(w)}.
\]
Using this notation, we have that the teachers are $\pi_{T_i}=\pi_{W_i}^T$.
Differentiating the above yields
\[
 A'(w)=\mathbb E_{\pi_w^T}[f],\qquad
 A''(w)=\operatorname{Var}_{\pi_w^T}(f).
\]
Substituting $\log\pi_w^T=wf-A(w)$ into entropy yields
\[
 H(\pi_w^T)=A(w)-wA'(w),\qquad
 \frac{d}{dw}H(\pi_w^T)
 =-wA''(w)
 =-w\operatorname{Var}_{\pi_w^T}(f)<0.
\]

Thus, $H_1<H_2$ if and only if $W_1>W_2$, as all of the above steps are reversible. Next, we recall the gradient from
\Cref{eq:cross-feature-gradient} as
\[
 d_i=C_i(t)\omega_i-F_t\theta_0
     =\rho W_i u-\theta_0.
\]
We therefore subtract the two
squared norms to obtain
\begin{align*}
 a_1^2-a_2^2
 &=\rho^2(W_1^2-W_2^2)-2\rho\beta(W_1-W_2)\\
 &=\rho(W_1-W_2)\bigl[\rho(W_1+W_2)-2\beta\bigr].
\end{align*}
If $W_1>W_2$, we have by \Cref{ass:confidence-beyond-student} that
\[
 \rho(W_1+W_2)-2\beta
 \geq\rho(W_1+W_2)-2\rho W_2
 =\rho(W_1-W_2)>0,
\]
which implies $a_1>a_2$, as desired.
\end{proof}

\begin{proof}[Proof of \Cref{thm:entropy-gradient-abridged}]
    From here, we can proceed with the argument the same way as \Cref{thm:learning-time-full} to finish the proof of the learning time separation.
\end{proof}

\section{Unifying Compatibility- and Confidence-Based Teacher Selection}
\label{app:alignment-confidence}

In this section, we detail the heuristic calculations we made in \Cref{sec:unified-routing}.
As before, we fix a prefix $s_t=(x,y_{<t})$. Recall from the previous sections that, under our assumptions, the gradient $d_i$ (and its corresponding norm $a_i$) with respect to teacher $i$ can be simplified to

\begin{equation}
\begin{aligned}
 d_i(t)&=\underbrace{\rho_i W_i u_i}_{\text{accessible teacher signal}}
                 -\theta_0,\\
 a_i^2&=\rho_i^2W_i^2+\|\theta_0\|_2^2
             -2\rho_iW_i\langle u_i,\theta_0\rangle,
\end{aligned}
\label{eq:alignment-confidence-gradient}
\end{equation}
where $u_i$ defines some unit direction in $\mathbb{R}^d$, and $W_i$ denotes the teacher head magnitudes. In particular, $\rho_i W_i u_i$ is the part of the teacher's logit
signal represented in the student's feature coordinates. Considering only that term in the gradient expression, we can make the following observations (which is also justified in our theory):
\begin{itemize}
    \item Closer student-teacher feature gap increases the $\rho_i$ towards 1, holding everything else fixed (e.g. the teacher head norms $W_i$). Indeed, in this case, \Cref{app:new-transfer-proofs} shows that this leads to bigger gradient norm.
    \item Lower entropy increases the teacher head norm $\|\omega_i\|_2$, holding all else fixed (e.g. the teacher features). Indeed, \Cref{app:confidence-routing} proves that this does lead to bigger gradient norm. 
\end{itemize}

\paragraph{Separating imitation from reward.}
The gradient comparison concerns how strongly a teacher moves the
student toward minimizing its KL objective with a given teacher. In the
performance-difference decomposition of \Cref{app:reward-decomposition},
this concerns imitation; the teacher's competence remains in its
separate reward-gap term. In general, it is an open question to quantify the trade-off between teacher competence and optimization pressures, but we would nonetheless like to include a brief discussion of success cases and failure modes of our two families of teacher selection.

To kick this discussion off, we remark that in most OPD settings, the teachers have some level of competence, and this is simply due to the fact that OPD in the recent literature generally distills from a model that was either 1) a post-trained version of the student, whether via RL or the likes, 2) a bigger (and hence more capable) model from the same model family as the student, or 3) both of the above. We will operate under this mental model for the rest of our exposition in this section.

\paragraph{On the confidence-based selection scheme.}
One obvious failure mode of this selection scheme is choosing a teacher that is confidently wrong. This may happen if, for instance the teacher was post-trained on bad data that perhaps caused model collapse. However, even barring this pathological case, another failure mode is simply the decrease in diversity of student responses due to imitation of a low-entropy distribution. This is a setting that we noticed in our experiments too by checking pass@$k$, in the sense that there were cases where although the entropy-routed model had better pass@1, it was worse on pass@$k$.

On the other hand, if the teacher is confidently right, then the teacher could potentially be a promising candidate for supervision. Nevertheless, if the thinking patterns do not match with the student, supervision can still remain difficult. The ideal case is where the teacher is confidently right, and the thinking patterns are compatible. In the context of our theory, this would correspond to $\psi=\phi$, with very large teacher head.

\paragraph{On the representation-based selection scheme.}
As with the entropy-based selection scheme, one failure mode is where the teacher (assumed to be post-trained version of the student or a variant) has been tilted towards some undesirable reward distribution. In this case, the student has no choice but to follow the teacher. Indeed, this is the big open problem in post-training, of how to calibrate competence vs. magnitude learning signal.

\section{Formulations of Various Objectives}
\label{app:obj-formulation}

\subsection{Implementation of OPD Objectives and Gradients}
\label{app:training-objectives}

Recall that a commonly used approximation to the full reverse KL gradient is the expression in \Cref{eq:opd-immediate-direction}, which ignores the terms of future time steps that a prefix may influence. For ease of exposition, we write the given gradient here once again:
\begin{align}
     \E_{x\sim\mathcal{D}, y\sim\pi_\theta(\cdot|x)} \qty[\sum_{t=1}^{|y|}  \log\frac{\pi_\theta(y_{t}|s_{t})}{\pi_T(y_{t}|s_{t})} \nabla_\theta \log \pi_\theta(y_t|s_t)].
\end{align}

Even in this semi-gradient approximation, we can implement it several ways. In our experiments from \Cref{sec:experiments}, we consider 1) sampled-token advantage, and 2) top-$k$ KL formulation. Previous works in the literature have explored the differences between the two \citep{fu2026revisiting, li2026rethinkingopd}, though in our setting, we are not concerned with the differences between the two implementations. We leave a more detailed analysis of these implementations to future work.

\paragraph{Sampled-token advantage.}
Let $\pi_T(\cdot|s_t)$ denote the routed teacher at prefix $s_t$: the selected teacher under hard routing, or the arithmetic mixture $\pi_{T,w}(\cdot|s_t)=\sum_i w_i(s_t)\pi_{T_i}(\cdot|s_t)$ under soft routing (in which case the log probability of the sampled token is $\log \sum_i w_i(s_t)\pi_{T_i}(y_t|s_t)$). Writing the token-level advantage from \Cref{sec:framework} as
\begin{align}
    A_t := \log\pi_T(y_t|s_t) - \log\pi_\theta(y_t|s_t),
    \label{eq:sampled-advantage-def}
\end{align}
the loss we implement is
\begin{align}
    \mathcal{L}_\text{sampled}(\theta) := -\E_{x\sim\mathcal{D}, y\sim\pi_\theta(\cdot|x)} \qty[\frac{1}{|y|}\sum_{t=1}^{|y|} \operatorname{sg}[A_t] \log \pi_\theta(y_t|s_t)],
    \label{eq:sampled-advantage-loss}
\end{align}
where $\operatorname{sg}[\cdot]$ is the stop-gradient operator. That is, the advantage (including its student log probability) is treated as a constant, and only the trailing $\log\pi_\theta(y_t|s_t)$ is differentiated, so that
\begin{align}
    \nabla_\theta \mathcal{L}_\text{sampled}(\theta) = -\E_{x\sim\mathcal{D}, y\sim\pi_\theta(\cdot|x)} \qty[\frac{1}{|y|}\sum_{t=1}^{|y|} A_t \nabla_\theta \log \pi_\theta(y_t|s_t)].
    \label{eq:sampled-advantage-gradient}
\end{align}
This is exactly \Cref{eq:opd-immediate-direction} with the routed teacher plugged in, up to the $1/|y|$ factor: each rollout contributes the average over its own length, and the outer expectation is estimated by averaging over the rollouts in the batch. Prompt and padding tokens are never included in either average. The routing weights and the teacher distribution are held fixed during the update, so the choice of teacher enters only through $A_t$.

\paragraph{Top-k KL.}
Let $\Omega_i(s_t)$ be the set of top-$K$ tokens of teacher $i$ at prefix $s_t$, where we take $K=8$ throughout, and let $\bar\pi_{T_i}(\cdot|s_t)$ be $\pi_{T_i}(\cdot|s_t)$ restricted to $\Omega_i(s_t)$ and renormalized to sum to one. The teacher target is
\begin{align}
    \bar\pi_{T,w}(v|s_t) := \sum_i w_i(s_t)\,\bar\pi_{T_i}(v|s_t), \qquad v \in \Omega(s_t) := \bigcup_{i:\, w_i(s_t)>0} \Omega_i(s_t),
    \label{eq:topk-target-def}
\end{align}
so each teacher is normalized on its own top-$K$ support before mixing, and the support of the target is the union of the supports of the teachers with positive weight. Let $\pi_\theta^{\Omega}(\cdot|s_t)$ denote the student's softmax restricted to $\Omega(s_t)$, i.e.\ renormalized over the same support. The loss is then the reverse KL between these two restricted distributions, averaged over the length of the rollout as before:
\begin{align}
    \mathcal{L}_\text{TopK}(\theta) := \E_{x\sim\mathcal{D}, y\sim\pi_\theta(\cdot|x)} \qty[\frac{1}{|y|}\sum_{t=1}^{|y|} \mathrm{KL}\qty(\pi_\theta^{\Omega}(\cdot|s_t) \,\|\, \bar\pi_{T,w}(\cdot|s_t))].
    \label{eq:topk-kl-loss}
\end{align}
In the same spirit as \Cref{eq:opd-immediate-direction}, we differentiate only through the per-token KL terms, with the sampled prefixes and the teacher target held fixed, which gives
\begin{align}
    \nabla_\theta \mathcal{L}_\text{TopK}(\theta) = \E_{x\sim\mathcal{D}, y\sim\pi_\theta(\cdot|x)} \qty[\frac{1}{|y|}\sum_{t=1}^{|y|} \sum_{v\in\Omega(s_t)} \pi_\theta^{\Omega}(v|s_t) \log\frac{\pi_\theta^{\Omega}(v|s_t)}{\bar\pi_{T,w}(v|s_t)}\, \nabla_\theta \log \pi_\theta^{\Omega}(v|s_t)].
    \label{eq:topk-kl-gradient}
\end{align}
Unlike the sampled-token advantage, which only ever sees the sampled token $y_t$, this takes a full expectation over the restricted support at each position, so no sampling noise enters at the token level; on the other hand, the student's probability mass outside of $\Omega(s_t)$ receives no direct signal. The same batch averaging conventions as above apply.

\subsection{Routing Scores and Teacher Targets}
\label{sec:score-definitions}

In this section, we give the precise score functions $S_i(s_t;\theta_j)$ that enter the router in \Cref{eq:router}, organized into the same three families as \Cref{sec:selection}, followed by the temperatures we use. In every case, the score is computed from the same prefix $s_t=(x,y_{<t})$ whose next token $y_t$ the loss supervises, so there is no lookahead into $y_t$, and the resulting weights $w_i(s_t)$ are detached and held fixed during the update. Given the weights, the teacher target is the selected teacher under hard routing ($\tau=0$) or the arithmetic mixture $\pi_{T,w}$ under soft routing, which then enters either the sampled-token advantage or the top-$K$ KL objective exactly as described in \Cref{app:training-objectives}.

\paragraph{Direct section routing.}
We use this in the two settings where the response has an explicit section structure. In the synthetic math--Triton task of \Cref{sec:synthetic-experiments}, the composite prompt asks for both solutions in a specified order and asks that each be wrapped in its own tag (\texttt{<math\_solution>} and \texttt{<triton\_solution>}). Routing is then deterministic given the tags completed so far in the prefix: before any tag is opened, we route to the teacher of the task requested first; inside an opened tag, we route to that domain's teacher, and once a tag is closed, we route to the teacher of the remaining task. Tags appearing inside fenced or inline code are ignored, and a token that completes a tag is still routed according to the state before it. In the Triton-only task of \Cref{sec:thinking}, the sections are the student's thinking block and the answer that follows it: every token generated while the thinking block is open is routed to the thinking teacher, and every token after the closing thinking tag is routed to Dr.Kernel, again with the closing tag itself routed according to the state before it. In the notation of \Cref{eq:router}, this is equivalent to taking $\tau=0$ with $S_i(s_t;\theta_j)=1$ for the assigned teacher and $0$ otherwise. 

\paragraph{Confidence-based routing.}
Following H-OPD \citep{yin2026hopdconfidenceawareheterogeneous}, the score of teacher $i$ is its negative entropy at the prefix,
\begin{equation}
    S_i^{\mathrm{ent}}(s_t;\theta_j) = -H\qty(\bar\pi_{T_i}(\cdot|s_t)),
    \label{eq:entropy-score}
\end{equation}
where $\bar\pi_{T_i}(\cdot|s_t)$ is the teacher's top-$K$ distribution renormalized on its own support, with $K=8$, as in \Cref{app:training-objectives}. In other words, this is the entropy of the teacher's top-8 renormalized distribution rather than its full-vocabulary entropy. Note that the score does not depend on the student at all.

\paragraph{Representation-based routing.}
We implement two scores in this family, depending on whether or not the student and teachers share an architecture.

\textit{Feature matching.} In the algorithmic coding setting of \Cref{sec:algo-coding}, the student and both teachers are Qwen3-4B models, so their hidden states live in the same space. Let $h_\theta(s_t)$ and $h_i(s_t)$ denote the student's and teacher $i$'s hidden states at the last position of the prefix, taken after the residual addition of the final decoder block (block 36) and before the final RMSNorm. We normalize each to unit length and score teacher $i$ by the negative squared Euclidean distance,
\begin{equation}
    S_i^{\mathrm{feat}}(s_t;\theta_j) = -\norm{\frac{h_\theta(s_t)}{\|h_\theta(s_t)\|_2} - \frac{h_i(s_t)}{\|h_i(s_t)\|_2}}_2^2,
    \label{eq:features-last}
\end{equation}
so that teachers whose feature direction is closer to the student's receive a higher score.

\textit{Hellinger distance.} In the synthetic math--Triton, visual reasoning, and Triton-only settings, at least one teacher has a different hidden size from the student (Dr.Kernel-8B, the Qwen3-VL-4B and Qwen3-4B teachers of the 2B student, and Qwen3-30B-A3B-Thinking, respectively), so we instead compare next-token distributions. For scoring, let $\Omega(s_t)$ be the union of \emph{both} teachers' top-8 supports at $s_t$ (so $|\Omega(s_t)|\le 16$), regardless of whether the eventual loss uses a hard-selected teacher. We gather the student's and each teacher's logits at every token in $\Omega(s_t)$ and renormalize there, writing $\pi_\theta^{\Omega}$ and $\pi_{T_i}^{\Omega}$ for the resulting distributions. The score is the negative Hellinger distance between the two,
\begin{equation}
    S_i^{\mathrm{Hel}}(s_t;\theta_j) = -H\qty(\pi_\theta^{\Omega}(\cdot|s_t), \pi_{T_i}^{\Omega}(\cdot|s_t)),
    \qquad
    H(\mu,\nu) = \qty[\frac12\sum_{v\in\Omega(s_t)}\qty(\sqrt{\mu(v)}-\sqrt{\nu(v)})^2]^{1/2},
    \label{eq:hellinger}
\end{equation}
which is the same definition as in \Cref{sec:selection}.

\paragraph{Routing temperatures.}
Soft weights are obtained by plugging the raw scores above into \Cref{eq:router}. For entropy and feature matching, we sweep $\tau\in\{0.1,0.25,0.5,1\}$, and for Hellinger routing, we sweep $\tau\in\{10^{-4},10^{-3},10^{-2}\}$, since Hellinger distances lie in $[0,1]$ and the gaps between two teachers are typically small. In each setting, we report the temperature with the best pass@1.

\section{Training Details and Hyperparameters}
\label{app:hyperparameters}

\paragraph{Algorithmic coding.}
The distillation prompts are the 25,276 Eurus-2-RL code training prompts \citep{cui2025prime}, with 1,024 code prompts held out for validation. These are the same prompts on which the code RL expert was trained. The math RL expert is trained on DeepMath \citep{he2025deepmath} for 500 iterations, following the domain-specialization setup of \citet{yang2026learningteachergeneralizedonpolicy}.

\paragraph{Synthetic math--Triton task.}
The Triton side starts from 11,621 KernelBook programs \citep{kernelbook2025}, which deduplicate by exact source code to 8,884 distinct programs. Before any length filtering, 512 programs are reserved for validation and 256 for a future test split; every remaining program is paired with a distinct, shuffled DeepMath training problem, and the math problems used for validation are excluded from math training. Composite prompts longer than 4,096 tokens are dropped, leaving 8,077 training pairs and 509 validation pairs.

\paragraph{Visual reasoning.}
The visual reasoning runs use the training set released with H-OPD \citep{yin2026hopdconfidenceawareheterogeneous}: 54{,}840 image-based problems with verifiable answers drawn from MMFineReason-123K and filtered by the authors with a GPT-4.1-mini judge. Each problem comes with a Qwen3-VL-generated description of its image; the vision-language teacher receives the image, while the text-only teacher receives the question with the description appended in place of the image. The student is Qwen3-VL-2B-Instruct and the teachers are Qwen3-VL-4B-Instruct and the text-only Qwen3-4B-Instruct-2507 \citep{qwen3vltechreport,qwen3techreport}.

\paragraph{Optimization and sampling.}
All four settings share the same recipe. The student is trained for 300 updates (400 for visual reasoning) with a batch of 128 prompts and a single rollout per prompt, sampled at temperature $1$ and top-$p=1$ with no top-$k$ restriction. We use AdamW with a constant learning rate of $10^{-6}$, no warmup, and gradient clipping at $1$. The prompt and response token caps are 2{,}048 and 8{,}192 for algorithmic coding, 4{,}096 and 16{,}384 for the synthetic math--Triton and Triton-only tasks, and 16{,}384 for both in visual reasoning. Thinking mode is disabled everywhere except in the Triton-only task. 

\section{Evaluations}
\label{app:eval-settings}

\paragraph{Computing pass@$k$.}
For every evaluation task, we draw $n$ independent responses at temperature $1$ and top-$p=1$ with no top-$k$ restriction, using seeds shared across all compared models, and count the number $c_i$ of responses that are correct for task $i$. With $N$ tasks, we report the unbiased estimator of \citet{chen2021code},
\begin{equation}
 \widehat{\operatorname{pass@}k}
 =\frac1N\sum_{i=1}^N
 \left[1-\frac{\binom{n-c_i}{k}}{\binom nk}\right],
 \label{eq:pass-at-k}
\end{equation}
so that every $k$ for a given model is computed from the same pool of $n$ samples. We report pass@1 and pass@16 for the synthetic math--Triton and Triton-only tasks, pass@1 and pass@64 for algorithmic coding, and pass@1 only for visual reasoning.

\paragraph{Synthetic math--Triton.}
The evaluation set pairs the 30 AIME 2025 problems released by OpenCompass \citep{aime2025opencompass} with the 100 KernelBench Level 2 tasks \citep{ouyang2025kernelbench}, giving 3{,}000 composite prompts, each with a fixed and balanced requested order of the two parts. We evaluate on the 3,000 pairs with $n=16$ responses per pair, a 16{,}384-token response cap, and thinking disabled. Let $M(r)$ indicate that response $r$ gives the correct math answer and $C(r)$ that its kernel passes all configured correctness trials (the fast$_0$ criterion of KernelBench). The joint score is $J_0(r)=M(r)\,C(r)$, so joint fast$_0$ requires both parts to be correct in the \emph{same} response; the math-portion and kernel-portion columns of \Cref{tab:synthetic-l2} report $M$ and $C$ separately.

\paragraph{Triton-only.}
We evaluate on the same 100 KernelBench Level 2 tasks with the same kernel prompt, drawing $n=16$ responses per task with a 32{,}768-token response cap so that the thinking-mode student can complete its reasoning. Correctness is fast$_0$ as above, and \Cref{tab:thinking-main} reports pass@1 and pass@16.

\paragraph{Algorithmic coding.}
We evaluate on the 164 HumanEval+ and 378 MBPP+ problems of EvalPlus \citep{chen2021code,austin2021program,liu2023evalplus}, drawing $n=128$ responses per problem. \Cref{tab:code-only-selected} reports pass@1 and pass@64 from this pool.

\paragraph{Visual reasoning.}
We report pass@1 only, on WeMath \citep{qiao2024wemath}, MathVista-mini \citep{lu2024mathvista}, and the reasoning questions of the CharXiv validation split \citep{wang2024charxiv}. WeMath is scored as question-level accuracy averaged over 16 sampled responses per question; MathVista-mini and CharXiv use a single sampled response per question. Final answers are extracted and graded with an LLM judge following the VLMEvalKit protocol used by H-OPD \citep{yin2026hopdconfidenceawareheterogeneous}; we use GPT-5.6 as the judge in place of GPT-4o-mini from the original paper.

\section{Additional Experiments and Analyses}
\label{app:figure-rollouts}

\subsection{Visual Reasoning}
\label{sec:visual}

We report our results for visual reasoning tasks in this section. In \Cref{tab:visual-main}, Hellinger routing achieves the highest
student pass@1 on WeMath \citep{qiao2024wemath} and MathVista-mini \citep{lu2024mathvista}: $68.69\%$ and
$67.40\%$, improving over the entropy-based routing by $0.67$ and $0.70$ percentage
points. It ties H-OPD on CharXiv RQ \citep{wang2024charxiv} at $33.20\%$ and outperforms
single-teacher OPD on all three benchmarks.

\begin{table}[!t]
\centering
\small
\caption{Visual reasoning question-level pass@1 (\%).}
\label{tab:visual-main}
\begingroup
\setlength{\tabcolsep}{3pt}
\begin{tabular}{lrrr}
\toprule
Method & WeMath & MathVista-mini & CharXiv RQ \\
\midrule
\multicolumn{4}{l}{\textit{Initial student model}} \\
Qwen3-VL-2B-Instruct & 60.97 & 62.70 & 29.10 \\
\midrule
\multicolumn{4}{l}{\textit{Teachers}} \\
Qwen3-VL-4B-Instruct & 76.52 & 74.20 & 39.30 \\
\midrule
\multicolumn{4}{l}{\textit{Without token-level routing}} \\
OPD (visual teacher) & 66.31 & 66.40 & 32.00 \\
\midrule
\rowcolor{routingrowcolor} \multicolumn{4}{l}{\textit{With token-level routing}} \\
\rowcolor{routingrowcolor} Entropy + TopK-KL ($\tau=0.5$) & 68.02 & 66.70 & \textbf{33.20} \\
\rowcolor{routingrowcolor} Hellinger + TopK-KL ($\tau=0.01$) & \textbf{68.69} & \textbf{67.40} & \textbf{33.20} \\
\bottomrule
\end{tabular}
\endgroup
\end{table}

\subsection{Code-only feature routing trace}
\label{app:figure-rollouts-code}

\begin{figure*}[!t]
\centering
\includegraphics[width=\textwidth,trim={0 136bp 0 153bp},clip]{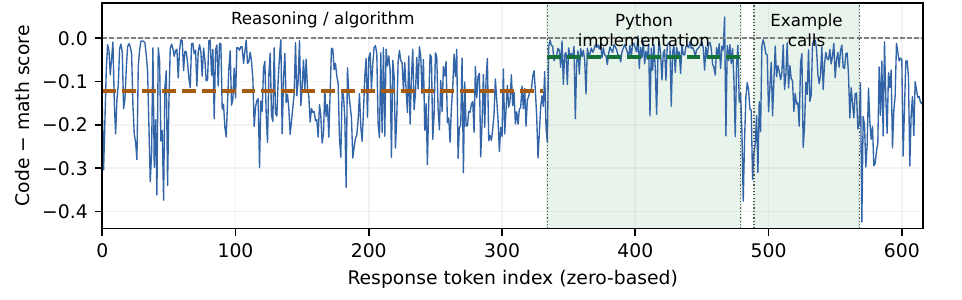}
\caption{Code-only final-layer feature routing ($\tau=0.1$; step 300). We plot the code-minus-math score gap here against response
token index. Here, the dashed colored segments show
the mean gaps during explanation and implementation.}
\label{fig:code-routing-probe}
\end{figure*}

In this section, we present the example rollout from the algorithmic coding task that was referred to in \Cref{fig:code-routing-probe,fig:code-rollout-colored}.

\begin{promptbox}\footnotesize
\begin{verbatim}
user
Find the largest rectangular area possible in a given histogram where the largest rectangle can be
    made of a number of contiguous bars. For simplicity, assume that all bars have the same width
    and the width is 1 unit, there will be N bars height of each bar will be given by the array arr.
Example 1:
Input:
N = 7
arr[] = {6,2,5,4,5,1,6}
Output: 12
Explanation: In this example the largest area would be 12 of height 4 and width 3. We can achieve
    this 
area by choosing 3rd, 4th and 5th bars.
Example 2:
Input:
N = 8
arr[] = {7 2 8 9 1 3 6 5}
Output: 16
Explanation: Maximum size of the histogram 
will be 8  and there will be 2 consecutive 
histogram. And hence the area of the 
histogram will be 8x2 = 16.
Your Task:
The task is to complete the function getMaxArea() which takes the array arr[] and its size N as
    inputs and finds the largest rectangular area possible and returns the answer.
Expected Time Complxity : O(N)
Expected Auxilliary Space : O(N)
Constraints:
1 <U+2264> N <U+2264> 10^{6}
0 <U+2264> arr[i] <U+2264> 10^{12}

Write Python code to solve the problem. Present the code in 
```python
Your code
```
at the end.
You need to think first then write the Python code.
assistant
<think>

</think>
\end{verbatim}
\end{promptbox}

\begin{responsebox}\footnotesize
\begin{verbatim}
To solve the largest rectangular area in a histogram problem efficiently in **O(N)** time and
    **O(N)** space, we can use a **monotonic stack** approach.

### Intuition:
For each bar, we want to find:
- The first bar to the left that is **shorter** than the current bar (left bound)
- The first bar to the right that is **shorter** than the current bar (right bound)

Then, the maximum area for that bar is `height * (right_bound - left_bound - 1)`.

This is because:
- The bar can extend from `left_bound + 1` to `right_bound - 1`.
- The width is `(right_bound - left_bound - 1)`, and the height is the height of the current bar.

### Steps:
1. Initialize an empty stack.
2. Iterate through the array:
   - While the stack is not empty and the current bar's height is less than the height of the bar
    at the top of the stack:
       - Pop the top of the stack and calculate the area.
       - The popped bar's height is the height of the rectangle.
       - The width is determined by the current index and the new top of the stack (or 0 if stack
    is empty).
   - Push the current index into the stack.
3. After the iteration, process any remaining elements in the stack in a similar way.

### Edge Cases:
- If all bars are increasing, the stack will be processed at the end.
- If all bars are decreasing, we process each bar as we go.

### Python Code:

```python
def getMaxArea(arr):
    stack = []
    max_area = 0
    n = len(arr)
    
    for i in range(n):
        while stack and arr[i] < arr[stack[-1]]:
            height = arr[stack.pop()]
            width = i if not stack else i - stack[-1] - 1
            max_area = max(max_area, height * width)
        stack.append(i)
    
    # Process remaining bars in the stack
    while stack:
        height = arr[stack.pop()]
        width = n if not stack else n - stack[-1] - 1
        max_area = max(max_area, height * width)
    
    return max_area
```

### Example Test Cases:

```python
# Example 1
print(getMaxArea([6, 2, 5, 4, 5, 1, 6]))  # Output: 12

# Example 2
print(getMaxArea([7, 2, 8, 9, 1, 3, 6, 5]))  # Output: 16
```

### Complexity:
- **Time Complexity:** O(N), as each bar is pushed and popped from the stack exactly once.
- **Space Complexity:** O(N), for the stack.

This solution is efficient and meets the constraints given.
\end{verbatim}
\end{responsebox}

\subsection{Triton-only entropy routing trace}
\label{app:figure-rollouts-triton}

In this section, we present the example rollout from the Triton-thinking task for which we plotted the teacher score gaps in \Cref{fig:thinking-routing-probe}.

\begin{promptbox}\footnotesize
\begin{verbatim}
user
You write custom Triton kernels to replace the pytorch operators in the given architecture to get
    speedups. 

    You have complete freedom to choose the set of operators you want to replace. You may make the
    decision to replace some operators with custom Triton kernels and leave others unchanged. You
    may replace multiple operators with custom implementations, consider operator fusion
    opportunities (combining multiple operators into a single kernel, for example, combining
    matmul+relu), or algorithmic changes (such as online softmax). You are only limited by your
    imagination.


        Here's an example to show you the syntax of inline embedding custom Triton kernels in
    torch: The example given architecture is: 

        ``` 

        import torch
import torch.nn as nn
import torch.nn.functional as F


class Model(nn.Module):
    def __init__(self) -> None:
        super().__init__()

    def forward(self, a, b):
        return a + b


def get_inputs():
    # randomly generate input tensors based on the model architecture
    a = torch.randn(1, 128).cuda()
    b = torch.randn(1, 128).cuda()
    return [a, b]


def get_init_inputs():
    # randomly generate tensors required for initialization based on the model architecture
    return []

        ``` 

        The example new arch with custom Triton kernels looks like this: 

        ```
        import torch
import torch.nn as nn
import torch.nn.functional as F
import triton
import triton.language as tl


@triton.jit
def add_kernel(
    x_ptr,  # Pointer to first input
    y_ptr,  # Pointer to second input
    out_ptr,  # Pointer to output
    n_elements,  # Total number of elements in input/output
    BLOCK_SIZE: tl.constexpr,
):
    # Each program handles a contiguous block of data of size BLOCK_SIZE
    block_start = tl.program_id(0) * BLOCK_SIZE
    # Create a range of offsets [0..BLOCK_SIZE-1]
    offsets = block_start + tl.arange(0, BLOCK_SIZE)
    # Mask to ensure we don't go out of bounds
    mask = offsets < n_elements
    # Load input values
    x = tl.load(x_ptr + offsets, mask=mask, other=0.0)
    y = tl.load(y_ptr + offsets, mask=mask, other=0.0)
    # Perform the elementwise addition
    out = x + y
    # Store the result
    tl.store(out_ptr + offsets, out, mask=mask)


def triton_add(x: torch.Tensor, y: torch.Tensor):
    """
    This function wraps the Triton kernel call. It:
      1. Ensures the inputs are contiguous on GPU.
      2. Calculates the grid (blocks) needed.
      3. Launches the Triton kernel.
    """
    assert x.is_cuda and y.is_cuda, "Tensors must be on CUDA."
    x = x.contiguous()
    y = y.contiguous()

    # Prepare output tensor
    out = torch.empty_like(x)

    # Number of elements in the tensor
    n_elements = x.numel()
    BLOCK_SIZE = 128  # Tunable parameter for block size

    # Determine the number of blocks needed
    grid = lambda meta: ((n_elements + meta["BLOCK_SIZE"] - 1) // meta["BLOCK_SIZE"],)

    # Launch the Triton kernel
    add_kernel[grid](x, y, out, n_elements, BLOCK_SIZE=BLOCK_SIZE)
    return out


class ModelNew(nn.Module):
    def __init__(self) -> None:
        super().__init__()

    def forward(self, a, b):
        # Instead of "return a + b", call our Triton-based addition
        return triton_add(a, b)
        ``` 

        
    You are given the following architecture: 

    ```
    import torch
import torch.nn as nn


class Model(nn.Module):
    """Model.

    .. seealso::
        Lin, Tsung-Yi, et al. "Focal loss for dense object detection."
        Proceedings of the IEEE international conference on computer vision. 2017.

    Args:
        gamma (float): Value from 0 to 5, Control between easy background and hard ROI
            training examples. If set to 0, equivalent to cross-entropy.
        alpha (float): Value from 0 to 1, usually corresponding to the inverse of class frequency
    to address class
            imbalance.
        eps (float): Epsilon to avoid division by zero.

    Attributes:
        gamma (float): Value from 0 to 5, Control between easy background and hard ROI
            training examples. If set to 0, equivalent to cross-entropy.
        alpha (float): Value from 0 to 1, usually corresponding to the inverse of class frequency
    to address class
            imbalance.
        eps (float): Epsilon to avoid division by zero.
    """

    def __init__(self, gamma=2, alpha=0.25, eps=1e-07):
        super(Model, self).__init__()
        self.gamma = gamma
        self.alpha = alpha
        self.eps = eps

    def forward(self, input, target):
        input = input.clamp(self.eps, 1.0 - self.eps)
        cross_entropy = -(target * torch.log(input) + (1 - target) * torch.
            log(1 - input))
        logpt = -cross_entropy
        pt = torch.exp(logpt)
        at = self.alpha * target + (1 - self.alpha) * (1 - target)
        balanced_cross_entropy = -at * logpt
        focal_loss = balanced_cross_entropy * (1 - pt) ** self.gamma
        return focal_loss.sum()


def get_inputs():
    return [torch.rand([4, 4, 4, 4]), torch.rand([4, 4, 4, 4])]


def get_init_inputs():
    return []

    ```
    
Optimize the architecture named Model with custom Triton kernels! Name your optimized output
    architecture ModelNew. Output the new code in codeblocks. Please generate real code, NOT
    pseudocode, make sure the code compiles and is fully functional. Just output the new model
    code, no other text, and NO testing code! 


assistant
\end{verbatim}
\end{promptbox}

\begin{responsebox}\footnotesize
\begin{verbatim}
<think>
Okay, I need to optimize the given PyTorch model using Triton kernels. The model is a Focal Loss
    implementation. Let me understand what the model does and where Triton can help.

The forward pass computes Focal Loss as follows:
1. Clamp input to [eps, 1 - eps].
2. Compute cross_entropy = -(target * log(input) + (1 - target) * log(1 - input)).
3. logpt = -cross_entropy.
4. pt = exp(logpt).
5. at = alpha * target + (1 - alpha) * (1 - target).
6. balanced_cross_entropy = -at * logpt.
7. focal_loss = balanced_cross_entropy * (1 - pt) ** gamma.
8. Return the sum of the focal loss.

The target and input are both 4D tensors with shape [4,4,4,4] in the example. The expectation is to
    replace the PyTorch operations with Triton kernels to speed up.

First, I should identify which parts can be accelerated:
- The element-wise operations (clamp, log, exp, etc.) are all element-wise and could be fused into
    a single kernel to reduce memory traffic and kernel launches.
- The gradient (hence the need for Triton) is not required here; we are just performing forward
    pass.

Looking at the operations:
- clamp(input, eps, 1 - eps) is easy. Calculator can list output elements.
- log, log(1 - input), exp, (1 - pt) ** gamma, etc. are all element-wise. With Triton, we can write
    a single kernel that computes all these steps in one pass.

But not all element-wise operations are compatible with Triton. Wait, we need to compute several
    steps. Let's see.

The main computation is:
- input_clamped = clamp(input, eps, 1 - eps)
- logpt = - (target * log(input_clamped) + (1 - target) * log(1 - input_clamped))
- pt = exp(logpt)
- at = alpha * target + (1 - alpha) * (1 - target)
- balanced_cross_entropy = -at * logpt
- focal_loss = balanced_cross_entropy * (1 - pt) ** gamma
- sum(focal_loss)

Most of these are element-wise. So a single Triton kernel could compute all these elements, then
    return the sum.

However, Triton kernels specify the output. The extreme case is that the kernel should compute all
    the elements and then the GPU sum can be done via a reduction in PyTorch. But Tripit can only
    compute per-element outputs, so data will be on the GPU. The sum can be done in PyTorch with
    .sum(). But the computational steps can be fused.

Wait, the return is focal_loss.sum(). So the offset is sum over all elements. The process is
    element-wise on the tensor.

So the approach is to create a Triton kernel that computes each element of focal_loss, writes it to
    the output tensor, and then return the sum of that output tensor.

But the meet selection point is passing the target and input as inputs to the kernel. The buffer
    size is the number of elements in the input.

Let's outline the steps in the kernel:
- Compute N = input_numel.
- For each index i, compute:
  - input_clamped = clamp(input[i], eps, 1 - eps)
  - log_i = log(input_clamped)
  - log_1_minus = log(1 - input_clamped)
  - cross_entropy_i = -(target[i] * log_i + (1 - target[i]) * log_1_minus)
  - logpt_i = -cross_entropy_i
  - pt_i = exp(logpt_i)
  - at_i = alpha * target[i] + (1 - alpha) * (1 - target[i])
  - balanced_cross_entropy_i = -at_i * logpt_i
  - focal_loss_i = balanced_cross_entropy_i * (1 - pt_i) ** gamma
- Store focal_loss_i into output[i].

Then in Python, we: compute output = triton_focal_loss(input, target, alpha, gamma, eps), then
    return output.sum().

This way, we reduce the number of kernel launches and memory reads/writes.

Now, let's design the Triton kernel.

First, need to handle:
- input and target are 4D tensors, but we can treat them as 1D for the kernel (since the scalar
    operations are element-wise).
- eps, alpha, gamma are scalars.

In the kernel, we need to:
- warps = 1D grid. Each program handles a block of elements.
- Signatures for the Triton kernel.

Let me write the kernel.

Constants for the kernel:
- The number of elements per block (BLOCK_SIZE) can be 1024 or 2048. We'll pick 1024 as a common
    choice.

Given that user cloakares must run on<U+749E> half the ext<U+961C><U+4FEF> regulable. Wait, no
    squeeze. Let me think.

Wait, the<U+9A71><U+52A8> is size 'n_elements'. For the grid, we'll use ((n_elements + BLOCK_SIZE -
    1) // BLOCK_SIZE,).

Now, code for the kernel.

But how to handle differentiate sections (target and input) in the kernel.

Martins recommendation: use the pointers to input, target, output.

So kernel signature:
def focal_loss_kernel(input_ptr, target_ptr, output_ptr, n_elements, eps, alpha, gamma, BLOCK_SIZE:
    tl.constexpr):

Wait, Triton allows dynamic arguments (like eps, alpha, gamma) but they need to be PL_DIST because
    they are scalars. Wait, to pass floats, we can pass them as scalars. Triton kernal sizes can be
    precomputed.

Wait, in Triton, the kernel can take scalar parameters as constants, but they need to be specified
    in the compile-time (tl.constexpr). But in the example, they are not. Wait, no, Triton supports
    passing floats as arguments to the kernel. Wait, the @triton.jit function accepts arguments.
    For scalar arguments, they can be passed as python floats. So:
def focal_loss_kernel(input_ptr, target_ptr, output_ptr, n_elements, eps, alpha, gamma, BLOCK_SIZE:
    tl.constexpr)

We can set eps, alpha, gamma as constants. Wait, but they are varies per instance. But DAO router
    blue. Wait, like gamma is part of the model. So we need to make them available to the kernel.
    Parameters to the kernel should be the first few arguments. Thertype method. Or we can define
    them as a pointer? No, they are scalars.

Wait, Triton supports passing scalar arguments without. We can pass them as floats. So kernel can
    take eps (float), alpha (float), gamma (float) as arguments.

But in thekernel, they need to be embedded as constants. Wait, Triton has a lesser IsNot(of known
    type). Allow\Events.Parameters are real.

Wait, trial:

@triton.jit
def focal_loss_kernel(input_ptr, target_ptr, output_ptr, n_elements, eps, alpha, gamma, BLOCK_SIZE:
    tl.constexpr):
    pid = tl.program_id(0)
    offsets = pid * BLOCK_SIZE + tl.arange(0, BLOCK_SIZE)
    mask = offsets < n_elements

    # Load input and target
    input_values = tl.load(input_ptr + offsets, mask=mask, other=0.0)
    target_values = tl.load(target_ptr + offsets, mask=mask, other=0.0)

    # Clamp input
    clamped = tl.maximum(tl.minimum(input_values, 1.0 - eps), eps)

    # Compute log(input) and log(1 - input)
    log_input = tl.log(clamped)
    log_1_minus = tl.log(1.0 - clamped)

    # Compute cross_entropy_l = -(target * log(input) + (1 - target) * log(1 - input))
    cross_entropy_l = -(target_values * log_input + (1.0 - target_values) * log_1_minus)

    # logpt = -cross_entropy_l
    logpt = -cross_entropy_l

    # pt = exp(logpt)
    pt = tl.exp(logpt)

    # at = alpha * target + (1 - alpha) * (1 - target)
    at = alpha * target_values + (1.0 - alpha) * (1.0 - target_values)

    # balanced_cross_entropy = -at * logpt
    balanced_cross_entropy = -at * logpt

    # focal_loss = balanced_cross_entropy * (1 - pt) ** gamma
    # Note: (1 - pt) and gamma are scalars per element. Use tl.pow for exponentiation.
    # Ensure (1 - pt) is positive? If not, use tl.maximum(0, ...), but standard_CTL uses log slips.
    focal_loss = balanced_cross_entropy * tl.pow(1.0 - pt, gamma)

    # Store output
    tl.store(output_ptr + offsets, focal_loss, mask=mask)

Wait, But tl.pow is [IBZ], yes.valuable. Work for general.

Alternatively, since gamma is an exponent, TL direct defines struct Keml as main long.

But alpha and gamma are scalars for each element, so it's fine.

Now, the instantiation:

In the Python function,<U+65F6><U+95F4><U+8282><U+70B9><U+05DE><U+05E0><U+05D8>:

- Ensure input and target are contiguous on CUDA.
- Calculate n_elements = input.numel().
- Create output tensor of dtype float32 (autograd by PyTorch).
- Launch the kernel with grid size = (ceil(n_elements / BLOCK_SIZE),).
- Return output.sum().

Wait, but World War III: the model returns sum of focal_loss. So after consuming the input, we can
    just return output.sum().

Now, let's check if the kernel is correct.

<U+86E9>_CODE:

import torch
import torch.nn as nn
import triton
import triton.language as tl


@triton.jit
def focal_loss_kernel(input_ptr, target_ptr, output_ptr, n_elements, eps, alpha, gamma, BLOCK_SIZE:
    tl.constexpr):
    pid = tl.program_id(0)
    block_start = pid * BLOCK_SIZE
    offsets = block_start + tl.arange(0, BLOCK_SIZE)
    mask = offsets < n_elements

    # Load input and target
    input_values = tl.load(input_ptr + offsets, mask=mask, other=0.0)
    target_values = tl.load(target_ptr + offsets, mask=mask, other=0.0)

    # Clamp input
    clamped = tl.maximum(tl.minimum(input_values, 1.0 - eps), eps)

    # Compute log(input) and log(1 - input)
    log_input = tl.log(clamped)
    log_1_minus = tl.log(1.0 - clamped)

    # Compute cross_entropy_l = -(target * log(input) + (1 - target) * log(1 - input))
    cross_entropy_l = -(target_values * log_input + (1.0 - target_values) * log_1_minus)

    # logpt = -cross_entropy_l
    logpt = -cross_entropy_l

    # pt = exp(logpt)
    pt = tl.exp(logpt)

    # at = alpha * target + (1 - alpha) * (1 - target)
    at = alpha * target_values + (1.0 - alpha) * (1.0 - target_values)

    # balanced_cross_entropy = -at * logpt
    balanced_cross_entropy = -at * logpt

    # focal_loss = balanced_cross_entropy * (1 - pt) ** gamma
    focal_loss = balanced_cross_entropy * tl.pow(1.0 - pt, gamma)

    # Store result
    tl.store(output_ptr + offsets, focal_loss, mask=mask)


def triton_focal_loss(input, target, alpha=0.25, gamma=2., eps=1e-7):
    """
    Compute Focal Loss using Triton kernel.
    Assumes input and target are CUDA tensors.
    """
    assert input.is_cuda and target.is_cuda, "Tensors must be on CUDA."
    input = input.contiguous()
    target = target.contiguous()
    n_elements = input.numel()
    # Output tensor for focal loss per element
    output = torch.empty(n_elements, dtype=torch.float32, device=input.device)

    # Kernel launch
    BLOCK_SIZE = 1024  # Can be tuned
    grid = lambda meta: (triton.cdiv(n_elements, meta['BLOCK_SIZE']),)
    focal_loss_kernel[grid](input, target, output, n_elements, eps, alpha, gamma,
    BLOCK_SIZE=BLOCK_SIZE)

    # Return sum of the focal loss
    return output.sum()


class ModelNew(nn.Module):
    def __init__(self, gamma=2, alpha=0.25, eps=1e-07):
        super(ModelNew, self).__init__()
        self.gamma = gamma
        self.alpha = alpha
        self.eps = eps

    def forward(self, input, target):
        # Use Triton kernel to compute focal loss
        return triton_focal_loss(input, target, self.alpha, self.gamma, self.eps)

Wait, but in the original code, the intput and target regeneration lot parametrised as something.
    Make sure the kernel parameters are passed correctly.

Wait, in the original model, during the forward, they do input.clamp(self.eps, 1.0 - self.eps), but
    in the determination flair<U+52A8><U+529B><U+7535><U+6C60><U+4E00><U+81F4>. However, thelose
    kernel expresses the clamping in the Triton kernel. So it's correct.

The math is:

balanced_cross_entropy = -at * logpt
focal_loss = balanced_cross_entropy * (1 - pt) ** gamma

But root_YUV and/or thoughts:
-deps = sum of all elements, so cherry in Exodus VT.

But let's check if note: (1 - pt) can be zero or negative. PT since logpt is - cross_entropy. But
    pt = exp(-cross_entropy) always positive. So (1 - pt) could be negative. However, in Focal
    Loss, gamma is positive, even if (1 - pt) is negative, the exponential becomes relaxing taken
    care of by the kernel.

Regarding clamping: the original code clamp(input, eps, 1 - eps). But log(1 - input) requires 1 -
    input > 0, so input < 1. So clamping to [eps, 1 - eps] ensures 1 - input > 0, so log is
    defined. Good.

Now, check if heat pool alot < flotation current bit. Worksheet.

What's the Guard(eps) here? Let's see. PlattonHttps: current batch submit then host port.util
    may<U+0434><U+043E><U+0440><U+043E><U+0436> the user, but this is quantity in micro dt.

So expanding the code.

Wait, but in the original architecture, the inputs are torch.rand([4,4,4,4]) so shape (4,4,4,4).
    The forward pass does input.clamp(Self.eps, 1.0 - self.eps) which is include 4D tensor. The
    Triton kernel of lockdown.

But the tensor descriptor SSD direction SUM<U+0E40><U+0E09><U+0E1E><U+0E32><U+0E30>ICS iOS agent
    s<U+00E1>ng<U+4E2A><U+591A><U+5C0F><U+65F6>. The dimensions aren't immediately important as the
    kernel treats them as 1D. So we can pass the tensors as 1D.

But what productions? performance-wise: the block size makes sense. 1024 per block, logarithmic
    protocol. I can<U+3019>         accumulation positions. But when scratched there is minor
    swellness benchmark.

Now, SAREN in each browser NFS directory. Let's check if the code is correct.

Primary family structure:
- The kernel computes each focal_loss element.
- The Python function returns sum of the array.

Consider to NaN if log(1 - clamped) is incorrectly. Alternatively, fixing small eps, op/judgment
    harness.

Less likely with HD venom; the math is safe.

Now, refactor the code, delete the previouschemical$mail. Let's output a new ModelNew class with
    Triton kernel.

Here's the code:

```python
import torch
import torch.nn as nn
import triton
import triton.language as tl


@triton.jit
def focal_loss_kernel(input_ptr, target_ptr, output_ptr, n_elements, eps, alpha, gamma, BLOCK_SIZE:
    tl.constexpr):
    pid = tl.program_id(0)
    block_start = pid * BLOCK_SIZE
    offsets = block_start + tl.arange(0, BLOCK_SIZE)
    mask = offsets < n_elements

    # Load input and target
    input_values = tl.load(input_ptr + offsets, mask=mask, other=0.0)
    target_values = tl.load(target_ptr + offsets, mask=mask, other=0.0)

    # Clamp input to [eps, 1 - eps]
    clamped = tl.maximum(tl.minimum(input_values, 1.0 - eps), eps)

    # Compute log(input) and log(1 - input)
    log_input = tl.log(clamped)
    log_1_minus = tl.log(1.0 - clamped)

    # Compute cross_entropy_l = -(target * log(input) + (1 - target) * log(1 - input))
    cross_entropy_l = -(target_values * log_input + (1.0 - target_values) * log_1_minus)

    # logpt = -cross_entropy_l
    logpt = -cross_entropy_l

    # pt = exp(logpt)
    pt = tl.exp(logpt)

    # at = alpha * target + (1 - alpha) * (1 - target)
    at = alpha * target_values + (1.0 - alpha) * (1.0 - target_values)

    # balanced_cross_entropy = -at * logpt
    balanced_cross_entropy = -at * logpt

    # focal_loss = balanced_cross_entropy * (1 - pt) ** gamma
    focal_loss = balanced_cross_entropy * tl.pow(1.0 - pt, gamma)

    # Store result
    tl.store(output_ptr + offsets, focal_loss, mask=mask)


def triton_focal_loss(input, target, alpha=0.25, gamma=2., eps=1e-7):
    """
    Compute Focal Loss using Triton kernel.
    Assumes input and target are CUDA tensors.
    """
    assert input.is_cuda and target.is_cuda, "Tensors must be on CUDA."
    input = input.contiguous()
    target = target.contiguous()
    n_elements = input.numel()
    # Output tensor for focal loss per element
    output = torch.empty(n_elements, dtype=torch.float32, device=input.device)

    # Kernel launch
    BLOCK_SIZE = 1024  # Can be tuned
    grid = lambda meta: (triton.cdiv(n_elements, meta['BLOCK_SIZE']),)
    focal_loss_kernel[grid](input, target, output, n_elements, eps, alpha, gamma,
    BLOCK_SIZE=BLOCK_SIZE)

    # Return sum of the focal loss
    return output.sum()


class ModelNew(nn.Module):
    def __init__(self, gamma=2, alpha=0.25, eps=1e-07):
        super(ModelNew, self).__init__()
        self.gamma = gamma
        self.alpha = alpha
        self.eps = eps

    def forward(self, input, target):
        # Use Triton kernel to compute focal loss
        return triton_focal_loss(input, target, self.alpha, self.gamma, self.eps)
```

Notes:
- The Triton kernel computes the full element-wise Focal Loss and stores it in an output tensor,
    then the Python function returns the sum, which matches the original code.
- The clamp, log, exp, and reduction steps are all fused into a single kernel, reducing the number
    of memory reads/writes and kernel launches.
- Normal device check is done on the frontend.
- The kernel uses a 1D grid with a BLOCK_SIZE of 1024. This is a common choice for element-wise
    operations; you can tune it for your hardware.

This should compile and run on CUDA-equipped systems. If you need more optimization, you could
    consider further multithreading or block size tuning, but this is a solid first step toward
    Triton-powered Focal Loss implementation.
</think>

```python
import torch
import torch.nn as nn
import triton
import triton.language as tl


@triton.jit
def focal_loss_kernel(input_ptr, target_ptr, output_ptr, n_elements, eps, alpha, gamma, BLOCK_SIZE:
    tl.constexpr):
    pid = tl.program_id(0)
    block_start = pid * BLOCK_SIZE
    offsets = block_start + tl.arange(0, BLOCK_SIZE)
    mask = offsets < n_elements

    # Load input and target
    input_values = tl.load(input_ptr + offsets, mask=mask, other=0.0)
    target_values = tl.load(target_ptr + offsets, mask=mask, other=0.0)

    # Clamp input to [eps, 1 - eps]
    clamped = tl.maximum(tl.minimum(input_values, 1.0 - eps), eps)

    # Compute log(input) and log(1 - input)
    log_input = tl.log(clamped)
    log_1_minus = tl.log(1.0 - clamped)

    # Compute cross_entropy_l = -(target * log(input) + (1 - target) * log(1 - input))
    cross_entropy_l = -(target_values * log_input + (1.0 - target_values) * log_1_minus)

    # logpt = -cross_entropy_l
    logpt = -cross_entropy_l

    # pt = exp(logpt)
    pt = tl.exp(logpt)

    # at = alpha * target + (1 - alpha) * (1 - target)
    at = alpha * target_values + (1.0 - alpha) * (1.0 - target_values)

    # balanced_cross_entropy = -at * logpt
    balanced_cross_entropy = -at * logpt

    # focal_loss = balanced_cross_entropy * (1 - pt) ** gamma
    focal_loss = balanced_cross_entropy * tl.pow(1.0 - pt, gamma)

    # Store result
    tl.store(output_ptr + offsets, focal_loss, mask=mask)


def triton_focal_loss(input, target, alpha=0.25, gamma=2., eps=1e-7):
    """
    Compute Focal Loss using Triton kernel.
    Assumes input and target are CUDA tensors.
    """
    assert input.is_cuda and target.is_cuda, "Tensors must be on CUDA."
    input = input.contiguous()
    target = target.contiguous()
    n_elements = input.numel()
    output = torch.empty(n_elements, dtype=torch.float32, device=input.device)

    BLOCK_SIZE = 1024
    grid = lambda meta: (triton.cdiv(n_elements, meta['BLOCK_SIZE']),)
    focal_loss_kernel[grid](input, target, output, n_elements, eps, alpha, gamma,
    BLOCK_SIZE=BLOCK_SIZE)
    return output.sum()


class ModelNew(nn.Module):
    def __init__(self, gamma=2, alpha=0.25, eps=1e-07):
        super(ModelNew, self).__init__()
        self.gamma = gamma
        self.alpha = alpha
        self.eps = eps

    def forward(self, input, target):
        return triton_focal_loss(input, target, self.alpha, self.gamma, self.eps)
```
\end{verbatim}
\end{responsebox}

\end{document}